\documentclass[10pt,journal,compsoc]{IEEEtran}
\IEEEoverridecommandlockouts
\usepackage{cite}
\usepackage{amsmath,amssymb,amsfonts}
\usepackage{algorithmic}
\usepackage{graphicx}
\usepackage{textcomp}
\usepackage{xcolor}
\usepackage{subfigure}
\usepackage{hhline}
\usepackage{epstopdf}
\def\BibTeX{{\rm B\kern-.05em{\sc i\kern-.025em b}\kern-.08em
    T\kern-.1667em\lower.7ex\hbox{E}\kern-.125emX}}

\usepackage{tikz}

\usepackage{amsthm}
\newtheorem{definition}{Definition} [section]
\newtheorem{theorem}{Theorem}  [section]

\begin{document}

\title{Nothing Wasted: Full Contribution Enforcement in Federated Edge Learning
\thanks{This work is partially supported by the US NSF under grant CNS-2105004.}
\author{Qin~Hu, Shengling Wang, \textit{Senior Member, IEEE,} Zehui Xiong, Xiuzhen Cheng, \textit{Fellow, IEEE}}
\IEEEcompsocitemizethanks{
\IEEEcompsocthanksitem Qin Hu is with the Department of Computer and Information Science, Indiana University-Purdue University Indianapolis, USA. \protect \\ E-mail: qinhu@iu.edu
\IEEEcompsocthanksitem Shengling Wang is with the School of Artificial Intelligence, Beijing Normal University, China. \protect \\ E-mail: wangshengling@bnu.edu.cn
\IEEEcompsocthanksitem Zehui Xiong is with the Pillar of Information Systems Technology and Design, Singapore University of Technology and Design, Singapore. \protect \\ E-mail: zehui\_xiong@sutd.edu.sg
\IEEEcompsocthanksitem Xiuzhen Cheng is with the School of Computer Science and Technology, Shandong University, China. \protect \\ E-mail: xzcheng@sdu.edu.cn
}
}
\IEEEtitleabstractindextext{
\begin{abstract}
The explosive amount of data generated at the network edge makes mobile edge computing an essential technology to support real-time applications, 
calling for powerful data processing and analysis provided by machine learning (ML) techniques. In particular, federated edge learning (FEL) becomes prominent in securing the privacy of data owners by keeping the data locally used to train ML models. 
Existing studies on FEL either utilize in-process optimization or remove unqualified participants in advance. 
In this paper, we enhance the collaboration from all edge devices in FEL to guarantee that the ML model is trained using all available local data to accelerate the learning process. To that aim, we propose a \textit{collective extortion (CE)} strategy under the imperfect-information multi-player FEL game, which is proved to be effective in helping the server efficiently elicit the full contribution of all devices without worrying about suffering from any economic loss. Technically, our proposed CE strategy extends the classical extortion strategy in controlling the proportionate share of expected utilities for a single opponent to the swiftly homogeneous control over a group of players, which further presents an attractive trait of being impartial for all participants. 
Moreover, the CE strategy enriches the game theory hierarchy, facilitating a wider application scope of the extortion strategy. Both theoretical analysis and experimental evaluations validate the effectiveness and fairness of our proposed scheme.

\end{abstract}
}

\maketitle

\section{Introduction}

The ubiquitous deployment of Internet-connected mobile devices leads to the amount of data generated at the network edge increasing exponentially, fostering the transformative computing paradigm, namely mobile edge computing \cite{xiao2019edge}. According to a recent report, the global market of edge computing is \$3.6 billion in 2020 and is anticipated to reach \$15.7 billion by 2025 \cite{market}. Facilitated by faster networking technologies such as 5G, edge computing becomes promising to support real-time applications, which calls for vigorous data processing and analysis capability at the edge. 
Thanks to the explosive growth of artificial intelligence, edge computing becomes more intelligent via implementing machine learning (ML) algorithms to achieve various functions such as classification and prediction.

However, since the data generated at the edge devices may be highly sensitive to the end users, it might be inappropriate to deploy conventional centralized ML algorithms which need to physically collect all training data from the devices. Federated learning (FL), a representative of distributed ML, turns into the aptest for edge computing, based on which the edge server and all connected devices accomplish training the same ML model in a collaborative manner, and thus this paradigm is also termed~\emph{federated edge learning} (FEL)~\cite{khan2020federated,lim2021towards}. More specifically, no device explicitly uploads the generated data in FEL, but their data can still contribute to training the shared ML model by iterative local learning, global aggregating, and updating~\cite{liu2020federated}.

Within this collaboration system, the most challenging but critical issue is to guarantee that all participants cooperate tacitly. To fulfill this goal, two lines of research have been carried out, namely \textit{in-process} \cite{mills2020communication,jiang2019model,wang2019adaptive,zhu2019broadband,amiri2020machine,yang2020federated,ahn2019wireless,tran2019federated,xu2019elfish,prakash2020coded,abad2020hierarchical,zeng2020energy,yang2020age,yang2019scheduling,amiri2020update}     
and \textit{in-advance} \cite{nishio2019client,kang2019incentive,ye2020federated,zhan2020learning,zhan2020infocom,pandey2020crowdsourcing,le2021incentive} FEL optimization, with the former improving the FEL system performance via optimizing learning algorithms or communication configurations during the FEL process, while the latter achieving the desirable performance through designing effective schemes to better establish and maintain the FEL system by avoiding inefficiency before the FEL process begins. Usually, taking precautions can enhance the FEL system as a preparative, so in-advance optimization becomes more cost-efficient than checking for the leaks during the working process. The state-of-the-art accomplishes this objective via either \textit{device selection} \cite{nishio2019client,kang2019incentive,ye2020federated}, which directly filters out unqualified devices, or \textit{incentive mechanism design} \cite{zhan2020learning,zhan2020infocom,pandey2020crowdsourcing}, which relies on a strong assumption of perfect information in the Stackelberg game~\cite{nie2018stackelberg}. Nevertheless, in practice, we may not have enough devices that can afford the elimination, and the devices may not own the full knowledge about each other. 

In this paper, we consider that an edge server and multiple devices collaborate in an FEL process repeatedly to optimize user experience in the long run. The server is the coordinator in charge of the whole FL process  
while the devices 
contribute their local learning results to obtain the globally trained model as a compensation at the end of each FL round\footnote{We term the ``round'' in this paper as finishing a specific FL task and obtaining a well-trained ML model, instead of one time of local training in FL or an epoch in the traditional ML model training phase.}. Within the whole FL process, the local training is only visible to and manageable by individual devices, leaving the room for selfish behaviors of perfunctorily contributing to the FEL via training the ML model using partial local datasets. To suppress this phenomenon, we utilize the multi-player simultaneous game to model the interactions between the edge server and devices in an FEL system, where none of them has perfect information about others, and aim at eliciting the full contribution of devices from the perspective of the server, instead of intolerantly eliminating malicious devices. However, the tight coupling of action and utility in this game makes it a dilemma for the server to play against devices because recklessly changing behaviors can lead to the server a decreasing utility. This brings us a question: \textit{is it possible for the server to entice full contributions from the devices without concerning about its utility loss?}

To answer this question, we resort to the extortion scheme which was first introduced as a special form of the zero-determinant (ZD) strategy \cite{press2012iterated}. By employing the extortion strategy, any player can independently control the proportion between the expected utility of itself and that of the opponent, which implies the potential to help the server control the utility in playing against devices. Nonetheless, the classical extortion strategy is derived for the two-player game, which is not applicable to our problem involving multiple players. In addition, it is clearly not efficient to directly carry it out between the server and every device in a one-by-one manner. To address this challenge, we put forward a \textit{collective extortion (CE)} strategy, which can achieve the goal of effortlessly controlling the overall utility of all devices with only one-time setting for the server. What's important, we comprehensively analyze the potential of the proposed CE strategy on enforcing the full cooperation of the devices, and further validate that it works impartially for all players with respect to utilities.

The main contributions are summarized as follows:
\begin{itemize}
\item We model the interactions between the edge server and devices in FEL as a multi-player simultaneous game, based on which, for the first time, we derive the powerful CE strategy to efficiently control the relative utility proportion between the CE adopter and a group of opponents.
\item The proposed CE strategy can not only effectively suppress the selfish behaviors of devices in FEL via enforcing their full contributions, but also enrich the theoretical system of game theory through extending the original two-player extortion strategy to the multi-player situation, and thus enlarging its application scope.
\item We demonstrate the effectiveness and fairness of the proposed CE strategy on driving the full cooperation of the devices with both theoretical analysis and experimental evaluations, which benefits the long-term system stability and liveness.
\end{itemize}


The rest of this paper comprises the following six sections. Section \ref{sec:related} investigates the most related work in improving FEL performance and Section \ref{sec:formulation} introduces our problem formulation. 
In Section \ref{sec:extortion}, we deduce the CE strategy for the multi-player situation,  followed by the analysis on its potential to enforce the full contribution from the devices in Section \ref{sec:mechanism}. Experimental evaluations are presented in Section \ref{sec:evaluation}. And we conclude this paper in Section \ref{sec:conclusion}.

\section{Related Work}\label{sec:related}
Existing research focusing on enhancing the overall system performance of FEL can be classified into \textit{in-process} 
and \textit{in-advance} 
optimization, depending on whether the operation steps lie in the FEL process or before that.

For the in-process optimization, researchers tried to improve the FEL performance via \textit{designing advanced learning algorithms} \cite{mills2020communication,jiang2019model,wang2019adaptive,zhu2019broadband,amiri2020machine,yang2020federated,ahn2019wireless,tran2019federated,xu2019elfish,prakash2020coded} or \textit{optimizing communication configurations} \cite{abad2020hierarchical,zeng2020energy,yang2020age,yang2019scheduling,amiri2020update}.       
In \cite{mills2020communication}, Mills \textit{et al.} proposed an adapting FedAvg algorithm based on the Adam optimization, which overcomes the shortcoming of the original FedAvg  with longer convergence time in dealing with the non-independent identically distributed data generated in internet-of-things (IoT). 
Considering about the constrained resources of edge devices, Jiang \textit{et al.} \cite{jiang2019model} proposed a scheme named PruneFL to adaptively adjust the model size for reducing training cost while maintaining comparable accuracy with the full model. 
To better control the global aggregation frequency in edge computing with limited resources, Wang \textit{et al.} \cite{wang2019adaptive} theoretically analyzed the gradient descent convergence bound. Leveraging on the over-the-air computation, several studies \cite{zhu2019broadband,amiri2020machine,yang2020federated,ahn2019wireless} achieved more efficient FL aggregation by taking advantage of the superposition of signals in the wireless multiple access channel. 
Tran \textit{et al.} \cite{tran2019federated} considered the trade-off between computation and communication latency and that between learning time and energy consumption in FL for wireless networks via solving a non-convex optimization problem. 
To deal with the straggler concern in FEL, a framework named ELFISH was proposed in \cite{xu2019elfish} to achieve resource-aware learning via dynamically masking computation-intensive neurons, while Prakash \textit{et al.} designed CodedFedL \cite{prakash2020coded} based on coded computing to inject structured redundancy in FL to compensate the negative impacts of straggling updates.
On the other hand, aiming to facilitate the FEL from the perspective of communications, optimal resource allocation was investigated in \cite{abad2020hierarchical,zeng2020energy,lim2021decentralized} and various transmission scheduling policies were designed in \cite{yang2020age,yang2019scheduling,amiri2020update,ng2020joint}.

For the in-advance FEL performance optimization, there are several recent studies which mainly focus on \textit{device selection} \cite{nishio2019client,kang2019incentive,ye2020federated,kang2020reliable} and \textit{incentive mechanism design} \cite{zhan2020learning,zhan2020infocom,pandey2020crowdsourcing,le2021incentive,lim2020hierarchical}. In \cite{nishio2019client}, to achieve the best learning result, a novel protocol was devised to select qualified devices according to their computational resource and communication conditions. In \cite{kang2019incentive}, Kang \textit{et al.} proposed a reputation based mechanism for screening out reliable devices to obtain high-quality model updates in FEL using the contract theory. 
To facilitate vehicular edge learning, selectively collecting good local model updates was considered in \cite{ye2020federated} using the two-dimension contract theory. 
Besides, Zhan \textit{et al.} designed deep reinforcement learning (DRL) based incentive mechanisms for edge-based FL in \cite{zhan2020learning,zhan2020infocom}, where the optimal pricing strategy of the aggregator and the best contribution strategy of the participants can be derived based on the hierarchical Stackelberg game. 
While Pandey \textit{et al.} solved the incentive problem in FL with communication efficiency consideration using a crowdsourcing framework and the two-stage Stackelberg game for equilibrium analysis,  
Le \textit{et al.} studied the incentive mechanism design for FL in wireless scenario via an auction game.

Relying on the power of taking precautions in enhancing the FEL performance, one can find that the existing studies either rigidly filter out unsatisfied devices or assume the availability of perfect information to implement, which can be impractical as there is hardly redundant number of participants or full knowledge about each other in FEL. To overcome these shortcomings, we utilize the multi-player simultaneous game to model the interactions between the edge server and devices with nobody having perfect knowledge of others, and then design an effective CE strategy to enforce the full contribution of selfish devices with fairness guaranteed.

\section{Problem Formulation}\label{sec:formulation}
\subsection{System Model}
As illustrated in Fig. \ref{fig:fel}, we consider an FEL system consisting of one edge server, denoted as $s$, and a set of edge devices, denoted as $\mathcal{N}=\{d_1,\cdots,d_n\}$. The system aims at providing better services to end users via conducting collaborative machine learning based on the data generated by all edge devices. Specifically, we assume that the FEL is conducted in a round-by-round manner, where the round is defined as \textit{accomplishing a certain learning task with the objective of training a global ML model with good performance}. Each device joins a round of FEL task by contributing the locally learned results obtained through training the initial ML model using the local dataset for multiple iterations. As a compensation, the server, who works as the FEL coordinator, returns the final well-trained ML model to the participating devices once the current round of FEL task finishes.

However, some devices may behave selfishly by utilizing partial of his\footnote{We use ``he'' and ``she'' to respectively represent anyone of the devices and the server.} local data to conduct the local training of ML model, by which they can make extra profits, such as saving computational resources and using the rest of the data to further improve the final ML model only for themselves. 
This sort of malicious behavior comes to be difficult for the server to timely detect and prevent due to the following two reasons. First, the server has no access to the local datasets held by devices for directly acquiring their size information or training efforts; 
second, the data distribution of devices is usually skewed in FEL, making it impossible to infer the size information, either. 
In this case, the server may behave strategically via choosing to return or not return the final ML model to the devices, thus helping suppress the selfishness in an opportunistic way, which will be detailed in the next subsection. 

For better understanding, we summarize major notations used in the following sections in Table \ref{tab:notations}.
 
 \begin{figure}[htbp]
\centering
\includegraphics[width=0.38\textwidth]{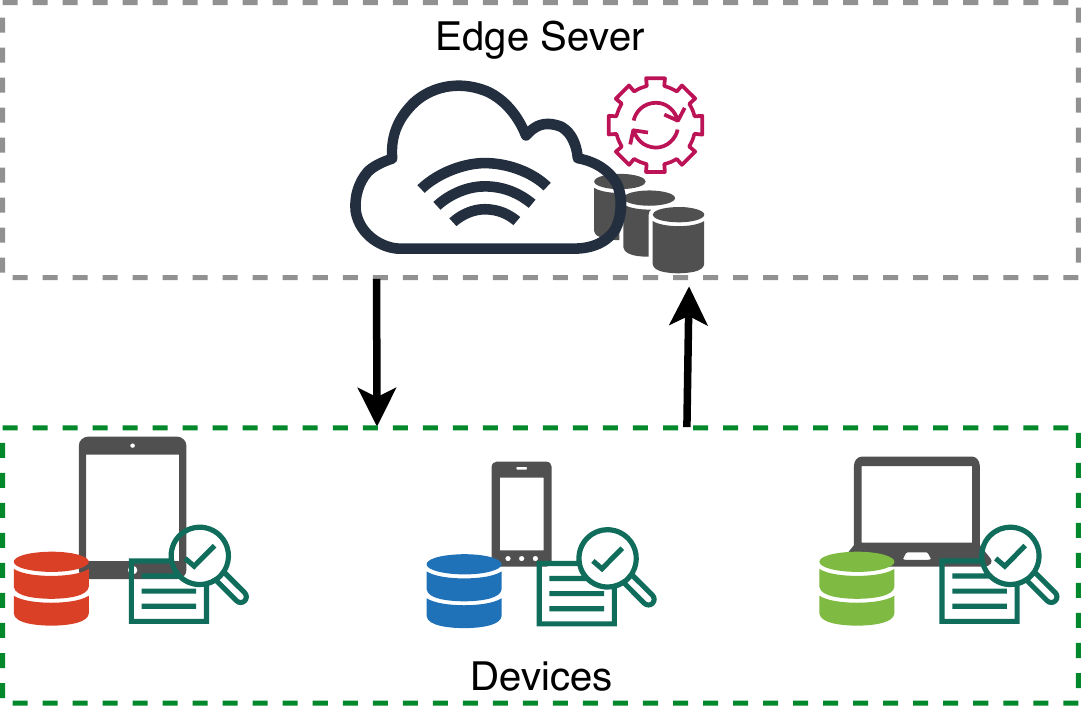}
\caption{The FEL system architecture.}
\label{fig:fel}
\end{figure}

\begin{table}[]
    \centering
    \footnotesize
    \caption{Summary of Notations.}
    \label{tab:notations}
    \begin{tabular}{| c | p{6.4cm} |}
    \hline
    Notation & Explanation \\
    \hline
    $x_i\in \{C,D\}$   & The action of the server playing against device $d_i$ \\
    \hline
    $y_i\in \{C,D\}$   & The action of any device $d_i$ \\
    \hline
    $u_s$ & The utility of the server \\
    \hline
    $u_i$ & The utility of any device $d_i$ \\
    \hline
    $\phi(\cdot)$ & The profit of the server \\
    \hline
    $b_i(\cdot)$ & The cost of the server sending the final model to $d_i$ \\
    \hline
    $\epsilon(\cdot)$ & The error of the final model \\ 
    \hline
    $\psi(\cdot)$ & The profit of the device $d_i$ \\
    \hline
    $m_i(\cdot)$ & The extra income of the device $d_i$ using partial data\\
    \hline
    \end{tabular}
\end{table}

\subsection{Game Formulation}
It is clear that neither the server nor any device can know the action of each other when they make their own decisions, which can be exactly modeled by a multi-player simultaneous game. Even though it seems that there are only two types of players, i.e., the edge server and the device, the number of players involved in the decision making and outcome witnessing of this game is multiple. In particular, the number of devices playing against the server in this FEL scenario can be large, and every device has his own preference on game strategy selection and operates with independent system parameters related to their benefits and costs. 

Formally, we define the server's action of returning the final ML model to the device as cooperation ($C$) and the action of not sharing the well-trained ML model as defection ($D$). For the device, we regard the action of conducting local learning using the full local dataset in a round of FEL as cooperation ($C$), while the behavior of employing only partial local data for FEL training can be viewed as defection ($D$). 
For clarity, we utilize $x_i$ to denote the action of the server playing against device $d_i$ and $y_i$ to express the action of device $d_i$ in this game. Thus, we have $x_i,y_i \in \{C,D\}$, where $i\in \{1,2,\cdots,n\}$.

It is worth noting that in the case of $y_i=D$, the specific amount of data utilized by each device $d_i$ during the FEL process can be heterogeneous from other peering devices. 
Here we treat any selfish behavior of not fully using the local data for model training as defection no matter how severe or slight this malicious action is. This qualitative consideration makes it easy for us to focus more on the elimination of devices' undesirable activities in the subsequent quantitative modeling and algorithm design sections.

Given the above actions, we can define the utility function of device $d_i$ as
\begin{align}\label{eq:device_utility}
u_i = \alpha_i \psi_i(x_i) + \beta_i m_i(y_i),
\end{align}
where $\alpha_i,\beta_i > 0$ are scale parameters, $\psi_i(x_i)$ is the profit the device can obtain according to the server's action of whether or not returning the final model, and $m_i(y_i)$ represents the extra income that the device can make by not fully using his local data to train the model, such as the spared computation, communication, and energy resource consumption. 

Here we have $\psi_i(C) > \psi_i(D)$ because the server's returned final model can enable the device to provide more efficient service to the end user so as to increase the user's satisfactory degree, which can be regarded as a higher payoff for the device. 
For easy expression, we use $\overline{\psi_i}$ and $\underline{\psi_i}$ to respectively represent $\psi_i(C)$ and $\psi_i(D)$. 
Considering that the cooperation action of contributing to FEL based on the full dataset leaves no extra room for the device to make more profit, we assume $m_i(C) = 0$. For the selfish behavior of using only partial local data for training the ML model, with $\delta_i \in [0,1)$ denoting the percentage of device $d_i$'s dataset contributed to FEL\footnote{As $\delta_i$ is a parameter related to the personal preference of each device regarding being selfish, here we assume that $\delta_i$ is a relatively stable value, not fluctuating drastically in the game rounds, which can be approximately estimated by the edge server through historical behaviors.}, we can define $m_i(D) = \lambda_i (1-\delta_i)$, where $\lambda_i$ is a device-dependent positive constant indicating the heterogeneity of devices.


Next we define the utility of the server as
\begin{align}\label{eq:server_utility}
u_s = \alpha_s \phi (\mathbf{y}) - \beta_s \sum_{i=1}^n b_i(x_i),
\end{align}
where 
 $\alpha_s,\beta_s > 0$ are scalars; $\phi (\mathbf{y})$ refers to the profit of the server gained from this round of FEL with the globally trained ML model and $\mathbf{y} = (y_1,y_2,\cdots,y_n)$ denotes the action vector of the devices; $b_i(x_i)$ is the cost of the server to send device $d_i$ the final trained model. Since the final model returned to all devices is the same, the main cost of sending it to every device is assumed to be the same as an example here\footnote{For different costs of the server to send the final model to devices, the overall research methodology proposed in this paper can still be applied although the derivation details may vary.}, with $b_i(C)=\frac{\rho}{n}$, where $\rho$ as a positive scalar denotes the overall cost of the server, and $b_i(D)=0$. 
 
The profit of the server obtained from the final model can be relatively complicated to depict, which is generally dependent on the specific ML model trained in the FEL system. In this paper, taking the convolutional neural network (CNN) based classifier as an example, we can describe $\phi (\mathbf{y})$ as follows:
\begin{align}\label{eq:profit}
\phi (\mathbf{y}) = & \frac{w}{1+\exp(r \varepsilon (\mathbf{y}) - t)},
\end{align}
where $w,r,t$ are positive scalars, and $\varepsilon(\mathbf{y})$ represents the classification error of the final trained model, jointly determined by the actions of all devices. Specifically, the server's profit $\phi$ reaches the maximum if $\varepsilon (\mathbf{y})$ approaches zero; and if the error is too large, $\phi$ becomes very small. 
Inspired by the power-law function proposed in \cite{chen2018my,johnson2018predicting}, we can define an exemplary $\varepsilon(\mathbf{y})$ as
\begin{align}
\label{eq:error}
\varepsilon(\mathbf{y}) = k (\sum_{y_i = C} F_i + \sum_{y_i=D} \delta_i F_i )^{-a}.
\end{align}
In the above equation, 
$F_i$ denotes the data size of device $d_i$; and $k,a \geq 0$ are tuning scalars to depict the non-linear relationship between the classification error and the training data size, where the larger the total data size used for training, the smaller the error. 
Combining \eqref{eq:profit} and \eqref{eq:error}, one can find that the less the number of defective devices, the larger the effective global training dataset, the smaller the error, which results in the larger profit for the server. In the extreme case where all devices choose $C$ (or $D$), $\varepsilon$ can reach the minimum (or maximum), and accordingly, $\phi$ turns to be the maximum (or minimum), denoted as $\overline{\phi}$ (or $\underline{\phi}$).  

Note that for other ML model training tasks in FEL, we may propose different formulas to describe the profit function $\phi (\mathbf{y})$, but its main characteristics about all cooperative devices producing $\overline{\phi}$ while all defective devices leading to $\underline{\phi}$ will generally hold. Therefore, the overall analysis framework, as well as the subsequent full contribution enforcement scheme, can still work in a similar way.

\begin{theorem}\label{thrm:condition}
The FEL system can form to function only when $\alpha_i (\overline{\psi_i}-\underline{\psi_i})>\beta_i \lambda_i(1-\delta_i)$ and $\alpha_s (\overline{\phi}-\underline{\phi}) > \beta_s \rho$.
\end{theorem}
\begin{proof}
To ensure that such an FEL system comprising one server and multiple devices functions well, the basic requirement is that all-cooperation behaviors can make it more beneficial than the case of all defection for any player. Otherwise, there is not enough motivation for any device or server to collaboratively participate in this FEL.

For device $d_i$, the utility of the all-cooperation case is $\alpha_i \overline{\psi_i}$ and that of all-defection is $\alpha_i \underline{\psi_i}+\beta_i \lambda _i(1-\delta_i)$. The above requirement leads to $\alpha_i \overline{\psi_i} > \alpha_i \underline{\psi_i}+\beta_i \lambda_i(1-\delta_i)$, which is equivalent to $\alpha_i (\overline{\psi_i}-\underline{\psi_i})>\beta_i \lambda_i(1-\delta_i)$. 

Similarly, for the server, the utility with cooperation actions from all players is $\alpha_s \overline{\phi} - \beta_s \rho$, while  all defection results in the utility of $\alpha_s \underline{\phi}$. Thus the FEL system requires that $\alpha_s \overline{\phi} - \beta_s \rho > \alpha_s \underline{\phi}$, which equals $\alpha_s (\overline{\phi}-\underline{\phi}) > \beta_s \rho$.
\end{proof}

Based on the above definitions of utilities, we can formally define an \textit{FEL game} as follows.
\begin{definition}[FEL Game]\label{def:FEL}
In the FEL system consisting of one server and $n$ devices, their interactions regarding whether to return the final model and whether to fully contribute to the learning process can be defined as a normal-form game $\mathcal{G} = (\{s\}\cup \mathcal{N}, \{C,D\}, \{u_s\}\cup \{u_i\})$  with $i \in \{1,\cdots,n\}$.
\end{definition}

\subsection{Dilemma in the FEL Game}
In fact, there exists a defection dilemma in the FEL game, which can be summarized in the following theorem.

\begin{theorem}\label{thrm:all_d}
In the FEL game defined in \ref{def:FEL}, $D$ is the best action for any player.
\end{theorem}
\begin{proof}
For any rational player, the best action can be derived by comparing the utility values under situations of choosing $C$ and $D$.
For any device $d_i$, the server's action $x_i$ being $C$ or $D$ clearly affects his utility, thus the device can consider these two cases separately. If $x_i = C$, his utility is $u_i = \alpha_i \psi_i(C) + \beta_i m_i(y_i)$, and since $m_i(C) < m_i(D)$, there exists $u_i(y_i=C) < u_i(y_i=D)$, which leads to his best action of $y_i = D$. If $x_i = D$, the device's utility becomes $u_i = \alpha_i \psi_i(D) + \beta_i m_i(y_i)$, where the function $m_i(\cdot)$ enforces the best action $D$ for the device again. In other words, no matter what action the server takes, the best action of the device is to defect.

Similarly, for the server, no matter what the action vector of the devices $\mathbf{y}$ is, the only factor affecting her utility that she can control is $x_i$. Referring to \eqref{eq:server_utility}, it can be concluded that only when the last item becomes zero can $u_s$ be maximized, which corresponds to $x_i = D$.
\end{proof}

According to Theorem \ref{thrm:all_d}, one can observe that the individual optimal action in the game among the server and the devices is always $D$, which means that the device always decides to take part in the FEL using partial dataset and the server never shares the final well-trained model to any device. This is obviously harmful for the overall benefit of the FEL system where the global model cannot be trained based on all generated data, leading to the reduced model performance. Thus, it becomes critical to solve this all-defection dilemma. 
Here we consider that the server is in charge of driving the cooperation from the devices due to the following two reasons. First, as the upper-level controller of the FEL system, the server hopes to obtain an optimal collaborative learning result, which becomes the motivation for her to get rid of this undesired situation; second, as the coordinator, the server can exert punishment to defective devices via not returning the final model, which indicates her capability to suppress malice. 

To elicit full contributions from the devices, one intuitive solution for the server is to design cooperation incentive schemes, which usually costs more for the server to entice profit-driven devices. Thus, it is imperative to design a new scheme embedded in this multi-player game process while preventing any interest loss for the server. Referring to \eqref{eq:server_utility}, one can observe that the utility of the server is collectively affected by the actions of all devices as well as herself. Thus, any reckless behavior change without a delicate plan would lead to undesired damage for the server, 
making it a critical challenge for the server to \textit{manage the behaviors of the devices without concerning her own utility.} 
Inspired by \cite{press2012iterated}, we find that the extortion mechanism, as a type of the zero-determinant (ZD) strategy, presents the merit of enabling the adopter to unilaterally control a proportional relationship between the expected utilities of two players, which implies the potential of helping solve the server's challenge. 

However, the conventional extortion strategy was originally developed for the two-player game, which is not directly applicable to our problem. Although one possible application idea is to carry it out between the server and each device, we can clearly notice the low efficiency of this one-by-one method. 
Thus, we resort to extending the extortion strategy to the multi-player scenario and name it as the \textit{collective extortion (CE) strategy}, which will be elaborated in the next section. 

\section{Collective Extortion Strategy}\label{sec:extortion}
As mentioned above, the classical extortion strategy derived in the two-player game cannot effectively fit in the FEL game scenario. In this section, we extend the two-player extortion strategy to the multi-player version, namely the CE strategy, which can solve the defection dilemma in the FEL game without suffering from the inefficiency of directly implementing the extortion strategy for each device. 

To be specific, we aim to enable the server to collectively control the overall utilities of all devices so as to further drive their cooperation behaviors, 
so here we set the action of the server playing against all devices to be homogeneous, denoted as $x$. 
Since there exist $n$ devices, 
the number of players in our FEL game is $n+1$ with each player choosing from two actions $C$ and $D$. And thus there exist $\eta = 2^{n+1}$ possible game results in total, which can be expressed as follows,
\begin{align*}
xy_1y_2\cdots y_n \in \{ \underbrace{ C \overbrace{CC\cdots C}^{n} }_{g_1}, \underbrace{ C \overbrace{CC\cdots}^{n-1}D }_{g_2}, \cdots, \underbrace{ D \overbrace{DD\cdots D}^{n} }_{g_{\eta}} \},
\end{align*}
where $g_i$ denotes the $i$-th game result.

In light of the conclusion in \cite{press2012iterated} that it is not disadvantageous for the short-memory player compared to the long-memory one, we assume that both the server and the devices have one-step memory and select their actions based on the game results in the last round. Thus, one can introduce the definitions of their mixed strategies as follows.

\begin{definition}[Mixed Strategy of the Server]
The server's mixed strategy is defined as $\mathbf{p} = (p_1,p_2,\cdots,p_{\eta})$ with $p_j$ denoting her conditional probability of choosing cooperation given the game result in the last round $g_j$.
\end{definition}

\begin{definition}[Mixed Strategy of the Device $d_i$]
The device $d_i$'s mixed strategy is defined as $\mathbf{q}^i = (q_1^i,q_2^i,\cdots,q_{\eta}^i)$ with $q_j^i$ denoting the conditional cooperation probability of device $d_i$ given the game result in the last round $g_j$.
\end{definition}

Accordingly, the defection probability of the server is $1-p_j$  and that of $d_i$ is $1-q_j^i$, where $j\in\{1,2,\cdots,\eta\}$. Then the Markov state transition matrix of this FEL game can be written as
$$\mathbf{M} = [M_{uv}]_{\eta \times \eta},$$ 
where the element $M_{uv}$ is the probability of transiting from the previous game result $g_u$ to the current one $g_v$ and can be defined as
\begin{align*}
M_{uv} = P \prod_{i=1}^n Q_i.
\end{align*}
In the above equation, $P$ and $Q_i$ are calculated according to
\begin{align*}
P = (p_u)^{z_0}(1-p_u)^{1-z_0}, \\
Q_i = (q_u^i)^{z_i}(1-q_u^i)^{1-z_i},
\end{align*}
where $z_0$ and $z_i,~i\in\{1,\cdots,n\},$ denote the actions of the server and device $d_i$ in the round with game result $g_v$, respectively. And they are assigned values according to 
\begin{align*}
z_0 = 
\begin{cases}
1, & x=C, \\
0, & x=D,
\end{cases} ~~~~
z_i = 
\begin{cases}
1, & y_i=C, \\
0, & y_i=D.
\end{cases}
\end{align*}
In other words, when the server's behavior is $x=C$ in the current game result $g_v$, it is the former part of $P$ functioning and thus $P=p_u$; otherwise, $P=1-p_u$. Next, $Q_i$ is derived in the same way according to the action of device $d_i$.

Then we define a non-negative vector $\mathbf{v}=(v_1,v_2,\cdots,v_{\eta})$ with the feature of $v_1+v_2+\cdots+v_{\eta}=1$, denoting the probability distribution over all possible game results in the stable state. Since $\mathbf{M}$ is the transition matrix, 
we know that when the Markov process reaches the stable state, there exists $\mathbf{v}\mathbf{M}=\mathbf{v}$, which equals $\mathbf{v}(\mathbf{M}-\mathbf{I})=\mathbf{v}\mathbf{M}'=0$ with $\mathbf{I}$ denoting the unit matrix and $\mathbf{M}'=\mathbf{M}-\mathbf{I}$. 

Let $\mathrm{Adj}(\cdot)$ and $\mathrm{det}(\cdot)$ be the adjugate and determinant operations on a matrix, respectively. 
According to the Cramer's rule, there exists $\mathrm{Adj}(\mathbf{M}')\mathbf{M}'=\mathrm{det}(\mathbf{M}')\mathbf{I}=0$. Comparing it with the above equation, one can conclude that $\mathbf{v}$ is proportional to every row of $\mathrm{Adj}(\mathbf{M}')$. Accordingly, the dot product of $\mathbf{v}$ and any vector $\mathbf{f}=(f_1,f_2,\cdots,f_{\eta})^T$ can be proportionally calculated by
\begin{align}\label{eq:vf}
\mathbf{v} \cdot \mathbf{f} \equiv \mathrm{det} [\mathbf{M}_1, \mathbf{M}_2, \cdots, \mathbf{M}_{\eta - 1}, \mathbf{f}],
\end{align}
where $\mathbf{M}_i,~i\in\{1,\cdots,\eta-1\}$, denotes the $i$-th column of $\mathbf{M}'$, and ``$\equiv$'' represents the proportional relationship.

Next, in light of the fact that the elementary transformation on any matrix  does not change its determinant value, we conduct column transformations on the matrix in \eqref{eq:vf}. More specifically, we first locate that the $\frac{\eta}{2}$-th column of this matrix refers to the game result of the server's cooperation and all devices' defection, i.e., $xy_1y_2\cdots y_n = CDD \cdots D$, and the $i$-th element in this column $\mathbf{M}_\frac{\eta}{2}$ can be expressed as $M_{i \frac{\eta}{2}} = p_i \prod_{j=1}^n (1-q_i^j)$; when adding all columns before the $\frac{\eta}{2}$-th column to $\mathbf{M}_\frac{\eta}{2}$, we obtain the new form of this column as follows:
\begin{align*}
\mathbf{M}^*_\frac{\eta}{2} = [p_1-1, p_2-1, \cdots, p_{\frac{\eta}{2}}-1, p_{\frac{\eta}{2}+1}, \cdots, p_{\eta}]^T.
\end{align*}
Then \eqref{eq:vf} can be written as
\begin{align*}
\mathbf{v} \cdot \mathbf{f} \equiv & \mathrm{det} [\mathbf{M}_1, \cdots, \mathbf{M}^*_\frac{\eta}{2} ,\cdots, \mathbf{M}_{\eta - 1}, \mathbf{f}] \\
= &  \mathrm{det} \left[
\begin{matrix}
p_1 \prod_{i=1}^{n} q^i_1-1 & \cdots &p_1-1& \cdots & f_1 &\\
\vdots                                  & \ddots &\vdots & \ddots & \vdots &\\
p_{\frac{\eta}{2}} \prod_{i=1}^{n} q^i_{\frac{\eta}{2}}& \cdots & p_{\frac{\eta}{2}}-1 &\cdots & f_{\frac{\eta}{2}} &\\
\vdots                                  & \ddots &\vdots & \ddots & \vdots &\\
p_\eta \prod_{i=1}^{n} q^i_\eta & \cdots &p_\eta & \cdots & f_\eta
\end{matrix} 
\right].
\end{align*}
It is clear that the $\frac{\eta}{2} $-th column is only related to the strategy of the server. Therefore, given any constant parameter $\gamma \neq 0$, the server can adjust the strategy $\mathbf{p}$ to meet the condition $\mathbf{M}^*_{\frac{\eta}{2} }= \gamma \mathbf{f}$ so as to achieve 
\begin{equation}\label{eq:zero_value}
\mathbf{v} \cdot \mathbf{f} = 0,
\end{equation}
because the $\frac{\eta}{2} $-th column and the last one of the matrix are proportional to each other.

In fact, the above proportional value can be converted to a real value by normalizing on the value of $\mathbf{v} \cdot \mathbf{1}$, where $\mathbf{1}$ denotes the all-one vector with the size of $\eta$. In particular, the expected utility of the server, denoted by $E_s$, and that of device $d_i$, denoted by $E_i$,  can be calculated by
\begin{align*}
E_s = \frac{\mathbf{v} \cdot \mathbf{u}_s}{\mathbf{v} \cdot \mathbf{1}} = \frac{\mathrm{det} [\mathbf{M}_1, \cdots, \mathbf{M}_{\eta - 1}, \mathbf{u}_s]}{\mathrm{det} [\mathbf{M}_1, \cdots, \mathbf{M}_{\eta - 1}, \mathbf{1}]}, \\
E_i = \frac{\mathbf{v} \cdot \mathbf{u}_i}{\mathbf{v} \cdot \mathbf{1}} = \frac{\mathrm{det} [\mathbf{M}_1, \cdots, \mathbf{M}_{\eta - 1}, \mathbf{u}_i]}{\mathrm{det} [\mathbf{M}_1, \cdots, \mathbf{M}_{\eta - 1}, \mathbf{1}]}, 
\end{align*}
where $\mathbf{u}_s = (u_s^1,\cdots,u_s^\eta)$ and $\mathbf{u}_i = (u_i^1,\cdots,u_i^\eta)$ are respectively the utility vector of the server and that of device $d_i$ following the same order of game results $(g_1,\cdots,g_\eta)$ and can be calculated according to \eqref{eq:server_utility} and \eqref{eq:device_utility}.

Next, we can derive the CE strategy as follows.
\begin{theorem}
By setting the strategy $\mathbf{p}$ to satisfy
\begin{align}\label{eq:extortion_strategy}
\mathbf{M}^*_{\frac{\eta}{2} }= \gamma \big( (\mathbf{u}_s - u_s^1 \mathbf{1} ) - \chi \sum_{i=1}^n ( \mathbf{u}_i -  u_i^1 \mathbf{1}) \big) ~(\gamma \neq 0),
\end{align}
the server can enforce an extortionate relationship between their expected utilities 
\begin{align}\label{eq:extortion_payoff}
E_s - u_s^1 = \chi \sum_{i=1}^n (E_i -  u_i^1),
\end{align}
with $\chi \geq 1$ being the extortion factor.
\end{theorem}

\begin{proof}
Given the expressions of $E_s$ and $E_i$, the server can enforce a zero value for any linear combination of the expected payoffs based on \eqref{eq:zero_value}. Particularly, if the server hopes to realize an extortionate share of expected utilities larger than the all-cooperation payoff $E_s - u_s^1 = \chi \sum_{i=1}^n (E_i -  u_i^1)$, the server can set $\mathbf{f} = (\mathbf{u}_s - u_s^1 \mathbf{1} ) - \chi \sum_{i=1}^n ( \mathbf{u}_i -  u_i^1 \mathbf{1})$ because the utility relationship is equivalent to $(E_s - u_s^1 ) - \chi \sum_{i=1}^n (E_i -  u_i^1) = 0$. 
Accordingly, we can know that the server's strategy $\mathbf{p}$ needs to comply with $\mathbf{M}^*_{\frac{\eta}{2} }= \gamma \mathbf{f} = \gamma \big( (\mathbf{u}_s - u_s^1 \mathbf{1} ) - \chi \sum_{i=1}^n ( \mathbf{u}_i -  u_i^1 \mathbf{1}) \big)$.
\end{proof}

With a feasible strategy satisfying the above condition, the server can unilaterally control to ensure that her own expected utility difference to $u_s^1$, i.e., the utility at all-cooperation state, is always $\chi$ times of the sum of all devices' expected utility differences to $u_i^1$. Based on the one-for-all feature of this strategy, we name it the \textit{collective extortion (CE) strategy}. In fact, CE not only expands the application scope of the original extortion strategy from the two-player game to the multi-player game, but also is effective to solve the problem of full contribution stimulation which will be elaborated in the next section.

It is worth noting that the base values in the CE strategy, i.e., the subtrahends $u_s^1$ and $u_i^1$ in \eqref{eq:extortion_payoff}, can be other values as long as the strategy $\mathbf{p}$ has feasible solutions to satisfy the corresponding condition similar to \eqref{eq:extortion_strategy}. For example, in the two-player game scenario, the original extortion strategy was proposed by using the payoffs at all-defection state as the base values \cite{press2012iterated}, where the feasibility of the extortion strategy was analyzed accordingly; while in \cite{hao2018payoff}, the range of base values are demonstrated to be between the payoffs of all-defection and all-cooperation.

\section{Full Contribution Enforcement Based on CE}\label{sec:mechanism}
As mentioned earlier, the server can fulfill an extortionate relationship between the expected utilities of herself and those of all devices via elaborately setting a CE strategy. In this section, we further explore the potential  of this strategy in stimulating full cooperation of the devices so as to solve our problem defined in Section \ref{sec:formulation}.

\subsection{Feasibility of the CE Strategy}\label{subsec:feasibility}
According to \eqref{eq:extortion_strategy}, one can get the server's strategy $\mathbf{p}$ as
\begin{align*}
p_j=
\begin{cases}
\gamma ( u_s^j - u_s^1 - \chi \sum_{i=1}^n (u_i^j - u_i^1) ) + 1,  j = 1,\cdots,\frac{\eta}{2}, \\
\gamma ( u_s^j - u_s^1 - \chi \sum_{i=1}^n (u_i^j - u_i^1) ),  j = \frac{\eta}{2}+1, \cdots, \eta.
\end{cases}
\end{align*}
Given a certain $\chi \geq 1$, its feasibility is dependent on the utility vectors of the server and the devices. Denote $A_j = u_s^j - u_s^1$ and $B_j = \sum_{i=1}^n (u_i^j - u_i^1)$, $j \in \{ 1,\cdots,\eta \}$.  Considering that $\gamma \neq 0$, the constraints of the utility vectors vary in the following two cases:

\textit{Case 1:} $\gamma > 0$.
\begin{align*}
\begin{cases}
-\frac{1}{\gamma} \leq A_j - \chi B_j \leq 0,  j = 1,\cdots,\frac{\eta}{2}, \\
0 \leq A_j - \chi B_j \leq \frac{1}{\gamma}, j = \frac{\eta}{2}+1, \cdots, \eta.
\end{cases}
\end{align*}

\textit{Case 2:} $\gamma < 0$.
\begin{align*}
\begin{cases}
0 \leq A_j - \chi B_j \leq -\frac{1}{\gamma}, j = 1,\cdots,\frac{\eta}{2}, \\
\frac{1}{\gamma} \leq A_j - \chi B_j \leq 0, j = \frac{\eta}{2}+1, \cdots, \eta.
\end{cases}
\end{align*}

\subsection{Potential of the CE Strategy to Drive Devices' Cooperation}\label{subsec:potential}
Under a feasible CE strategy $\mathbf{p}$ adopted by the server, one can analyze its potential of driving the devices to fully utilize their local datasets in FEL. To that aim, we first assume that each device in the FEL game searches for the best response strategy in an \textit{evolutionary} manner. The reason is that the device lacks the global game information compared to the server who can interact with all devices in the game. Here we assume that a device adjusts his strategy with the goal of improving his own utility regardless of the strategy or utility of the server. Inspired by \cite{smith1982evolution}, we define the following strategy evolving path for an evolutionary device\footnote{Here we discuss one of the devices as a representative and thus omit the subscript $i$ for brevity.} with $q^t \in [0,1]$ denoting his cooperation probability at round $t$,
 \begin{align}\label{eq:evolution}
 q^{t+1} = q^t \frac{W_C^t}{W^t},
 \end{align}
where 
$W_C^t$ refers to the expected utility of cooperation, and $W^t$ represents the total expected utility. With $W_D^t$ denoting the expected utility of defection, the total expected utility can be calculated by 
\begin{align}\label{eq:total_utility}
W^t = q^t W_C^t + (1-q^t) W_D^t.
\end{align}
Referring to the right side of \eqref{eq:evolution}, we can find that the numerator $q^t W_C^t$ is a part of the denominator, resulting in $ q^{t+1} \in [0,1]$.

To investigate whether the proposed CE strategy can drive the full cooperation of devices, we need to study the condition of $q^t$ increasing. 
According to \eqref{eq:evolution}, we can find that only when $W_C^t > W^t$ can the cooperation probability of the device increase in the next round. Combining it with \eqref{eq:total_utility}, we can derive the sufficient condition of the CE strategy being able to enforce the device become more cooperative as follows:
\begin{align*}
W_C^t > W^t & \Rightarrow W_C^t > q^t W_C^t + (1-q^t) W_D^t, \\
					& \Rightarrow W_C^t > W_D^t,
\end{align*}
for $q^t \neq 1$. In fact, in the case of $q^t =1$, there exists $W^t = W_C^t$ according to \eqref{eq:total_utility}  and thus $q^{t+1}$ is always 1, which never requires any function of the CE strategy. 

Therefore, in the following, we focus on solving the problem of \textit{when $q^t \in [0, 1)$, can the CE strategy function to elicit the cooperation from the device?} Referring to the above-derived sufficient condition, we can find that this problem turns to be \textit{whether the CE strategy can lead to $W_C^t > W_D^t$.}

Recalling the power of the CE strategy presented in Section \ref{sec:extortion}, the server's strategy works on the whole set of devices according to \eqref{eq:extortion_payoff} and \eqref{eq:extortion_strategy}. To study the effect of the CE strategy on any individual device, we consider two possible situations of devices in the FEL game:
\begin{enumerate}
\item[S1:] Devices are homogeneous using the same strategy and receiving the same utility;
\item[S2:] Devices are heterogeneous with various strategies and utilities.
\end{enumerate}

Then, for both situations, we can demonstrate that the device tends to cooperate under the server's CE strategy, which are respectively presented in the following theorems.

\begin{theorem}\label{thrm:cooperation_homo}
In the case of all devices with the same strategy and utility, the server utilizing the CE strategy can enforce any evolutionary device to obtain the cooperation probability $q^t\rightarrow 1$. 
\end{theorem}
\begin{proof}

For situation S1 where all devices involved in FEL are homogeneous, 
since everyone uses the same strategy and the server exerts one uniform strategy to all of them as well, we study the cooperation probability of any device here as an representative. 
According to \eqref{eq:extortion_payoff}, we can derive the expected utility of the device as
\begin{align}\label{eq:extor_proof}
E_i = \frac{1}{n \chi} (E_s - u_s^1) + u_i^1 .
\end{align}

Next, we analyze the expected utilities of the evolutionary device with different actions, i.e., $W_C^t$ and $W_D^t$. More specifically, when the device takes the cooperation action, the server's expected utility $E_s$ depends on her own action, where $C$ leads to $u_s^1 = \alpha_s \overline{\phi} - \beta_s \rho$ while $D$ results in $ u_s^{\frac{\eta}{2}+1} = \alpha_s \overline{\phi} $ according to \eqref{eq:server_utility}. 
Based on the above equation \eqref{eq:extor_proof}, one can find that the server's expected utility brings two possible payoffs for the device, which are 
\begin{align*}
E_i ({x=C\mid y_i = C}) & = u_i^1, \\
E_i ({x=D\mid y_i = C}) & = \frac{1}{n \chi} (u_s^{\frac{\eta}{2}+1}  - u_s^1) + u_i^1 \\
& = \frac{1}{n \chi} \beta_s \rho + u_i^1. 
\end{align*}
Assuming that the cooperation of the server at round $t$ is $p^t$, the expected cooperation payoff of the device can be calculated by
\begin{align}\label{eq:wc}
W_C^t =  & p^t E_i ({x=C \mid y_i = C}) \notag \\
& + (1-p^t) E_i ({x=D \mid y_i = C}). 
\end{align}

While the device chooses the action $D$, the expected utility of the server $E_s$ would become $u_s^{\frac{\eta}{2}} = \alpha_s \underline{\phi} - \beta_s \rho$ and $u_s^\eta = \alpha_s \underline{\phi} $ for $x=C$ and $D$, respectively. And according to \eqref{eq:extor_proof}, the device's payoff can be
\begin{align*}
E_i ({x=C\mid y_i = D}) & = \frac{1}{n \chi} (u_s^{\frac{\eta}{2}} - u_s^1) + u_i^1 \\
& = \frac{1}{n \chi} ( \alpha_s \underline{\phi} - \alpha_s \overline{\phi} ) + u_i^1, \\
E_i ({x=D\mid y_i = D}) & = \frac{1}{n \chi} (u_s^\eta  - u_s^1) + u_i^1 \\
& = \frac{1}{n \chi} ( \alpha_s \underline{\phi} - \alpha_s \overline{\phi} + \beta_s \rho) + u_i^1. 
\end{align*}
Thus, the expected defection payoff of the device turns to be
\begin{align}\label{eq:wd}
W_D^t = & p^t E_i ({x=C\mid y_i = D}) \notag \\
& + (1-p^t) E_i ({x=D\mid y_i = D}). 
\end{align}

Since $\underline{\phi} < \overline{\phi} $, there clearly exists $E_i ({x=C\mid y_i = D}) < E_i ({x=C \mid y_i = C}) $ and $E_i ({x=D\mid y_i = D}) < E_i ({x=D \mid y_i = C})$, which leads to $W_C^t > W_D^t$ by
comparing \eqref{eq:wc} and \eqref{eq:wd}, thus concluding the proof of the theorem.
\end{proof}

\begin{theorem}\label{thrm:cooperation_heter}
In the case of all devices with different strategies and utilities, the server's CE strategy can drive an evolutionary device to get $q^t\rightarrow 1$. 
\end{theorem}
\begin{proof}
Given the heterogeneous devices in situation S2, to focus on the behavior of (any) one specific device $d_i$, we assume that the strategies of other devices are given fixed, and thus their expected utilities are also certain values. To comply with \eqref{eq:extortion_payoff}, we denote 
$$\Delta_{-i} = \chi \sum_{j\in \mathcal{N} \setminus \{i\}} (E_j - u_j^1).$$
Then the expected utility of $d_i$ in this case turns to be
$$ E_i =  \frac{1}{\chi} (E_s-u_s^1-\Delta_{-i}) + u_i^1.$$

Similar to the proof of Theorem \ref{thrm:cooperation_homo}, we  can calculate $W_C^t$ according to $W_C^t =  p^t E_i ({x=C \mid y_i = C}) + (1-p^t) E_i ({x=D \mid y_i = C})$, where 
\begin{align*}
E_i ({x=C\mid y_i = C}) & = \frac{\Delta_{-i}}{\chi}  + u_i^1, \\
E_i ({x=D\mid y_i = C}) & = \frac{1}{\chi} (\beta_s \rho -\Delta_{-i}) + u_i^1. 
\end{align*}

For the calculation of $W_D^t = p^t E_i ({x=C\mid y_i = D}) + (1-p^t) E_i ({x=D\mid y_i = D})$, we have 
\begin{align*}
E_i ({x=C\mid y_i = D}) & = \frac{1}{\chi} ( \alpha_s \underline{\phi} - \alpha_s \overline{\phi} -\Delta_{-i}) + u_i^1, \\
E_i ({x=D\mid y_i = D}) & = \frac{1}{\chi} ( \alpha_s \underline{\phi} - \alpha_s \overline{\phi} + \beta_s \rho -\Delta_{-i}) + u_i^1. 
\end{align*}

Due to the same reason of $\underline{\phi} < \overline{\phi} $, we can obtain $E_i ({x=C\mid y_i = D}) < E_i ({x=C \mid y_i = C}) $ and $E_i ({x=D\mid y_i = D}) < E_i ({x=D \mid y_i = C})$ in this situation as well, resulting in $W_C^t > W_D^t$, which can lead to the gradual increase of $q^t$ until approaching to 1.
\end{proof}

From the above two theorems, we can tell that the CE strategy can theoretically incentivize the final cooperation of any device involving in the FEL game with an evolutionary mindset no matter in the homogeneous or heterogeneous device settings. In other words, devices can usually be driven to participate in the FEL process with fully using their local datasets and contributing to the global learning without any reservation. 

\subsection{Fairness of the CE Strategy}\label{subsec:fairness}

Given the vigorous force of the CE strategy in stimulating devices' collaboration, one may concern about \textit{what if the server behaves defectively via not returning the final well-trained model to the devices so as to save sending cost for obtaining a better utility?} 
This question will be investigated in detail as follows.

According to Theorems \ref{thrm:cooperation_homo} and \ref{thrm:cooperation_heter}, the final actions of the devices become cooperation as the number of game rounds increases. Nevertheless, the server can still select her action from $C$ and $D$. However, according to the following theorem, one can see that the best action for the server  with the CE strategy to keep the long-term stability is to eventually choose $C$.

\begin{theorem}
The final action of the server adopting the CE strategy is cooperation.
\end{theorem}
\begin{proof}
After enough number of FEL game rounds, devices choose cooperation eventually. Then the server's cooperation can bring the cooperative device $d_i$ the utility $u_i^1 = \alpha_i \overline{\psi_i}$ with the game result $g_1=CC\cdots C$, and her defection action can make the cooperative device obtain the utility $u_i^{\frac{\eta}{2}+1} = \alpha_i \underline{\psi_i}$ at the game result $g_{\frac{\eta}{2}+1}=DC\cdots C$. 

Referring to \eqref{eq:extortion_payoff}, one can find that the cooperative server forming the game state $g_1$ can still make it hold stably since the right side turns out to be zero with $E_i = u_i^1$ in the long run. However, if the server chooses to be defective constantly, the right side of \eqref{eq:extortion_payoff} would become negative because the device's utility in this case is $E_i = u_i^{\frac{\eta}{2}+1}$, which is less than $u_i^1$ due to 
$\underline{\psi_i} < \overline{\psi_i}$, 
and thus there exists $E_s < u_s^1$. This is clearly unfavorable for a reasonable server. Thus, the best action for the server is also cooperation in the long run.
\end{proof}

Based on the above theorem, we can conclude that our proposed CE strategy employed by the server is fair for all players, which would result in all cooperation and bring the same-level utility to the server and devices.

\section{Experimental Evaluation}\label{sec:evaluation}
In this section, we conduct a series of experiments to demonstrate the effectiveness of the proposed CE strategy in eliciting full cooperation from all devices in the FEL game and other attractive features mentioned in the previous section. The machine used for simulation experiments is a desktop computer with a 3.59 GHz 6-Core processor and 16 GB memory. 
In all experiments, we fix the number of devices $n=8$. 
Scalar parameters for devices are randomly set following uniform distributions with $\alpha_i \in [0, 3],\beta_i\in[0,2],\psi_i(C)\in(1,2],\psi_i(D)\in[0,1],m_i(C)=0$, and $m_i(D)\in(0,1]$. 

While for the server,  the parameter values independent of the ML model are firstly set as $\alpha_s=5,\beta_s=2,b_i(C)=1,b_i(D)=0$. 
To appropriately set the parameters related to the profit function $\phi(\mathbf{y})$ which is closely depending on the specific ML task, we utilize the  MNIST database \cite{lecun1998gradient} using 6,000 data samples to train a 2-layer CNN classifier, where each device is assumed to generate $F_i=$750 samples in non-iid 
manner. The obtained fitting parameters in \eqref{eq:error} are $k=13.2$ and $a=0.7$ with 95\% confidence, and $\delta_i=0.018$. Further, we fix $w=r=10,t=5$ in \eqref{eq:profit} and obtain the extreme values of $\phi(\mathbf{y})$ as $\overline{\phi}=10$ and $\underline{\phi}=5$. 
Note that we also test other sets of parameter values satisfying the requirements shown in Theorem \ref{thrm:condition} and Section \ref{subsec:feasibility}, but we obtain similar results which are omitted for brevity. Besides, each experiment is repeated 20 times to obtain the average for statistical confidence.

\subsection{Effectiveness of the CE strategy to enforce full cooperation}

To figure out whether the proposed CE strategy adopted by the server can enforce full cooperation from any evolutionary device as theoretically proved in Section \ref{subsec:potential}, we compare it with four classical strategies, namely ALLC (all cooperation), ALLD (all defection), TFT (tit-for-tat), and WSLS (win-stay-lose-shift). The first two strategies are easy to understand where the server stays constantly cooperative or defective. The TFT strategy means that the server behaves according to the device's previous action while in WSLS the server keeps on choosing an action if it brings a high utility and switches to the other action otherwise. 

Taking the first device as an example, we report the comparative experiment results in Fig. \ref{fig:ex1}, where his initial cooperation probability varies as $q^0=0.10,0.40,0.60, 0.90$ to indicate the robustness of our proposed CE strategy. It is clear that no matter how cooperative the device is at the beginning, the server adopting the CE strategy can elicit the final cooperation of the evolutionary device. As $q^0$ increases, the time consumption to achieve the stable state is less. This is because the more cooperative the device, the easier to drive his full cooperation.
It is clear that other strategies cannot achieve this goal as all of them lead to the cooperation probability approaching zero finally.

\begin{figure}[htbp]
\centering
\subfigure[$q^0=0.10$.]{
\label{fig:ex1_1}
\includegraphics[width=0.23\textwidth]{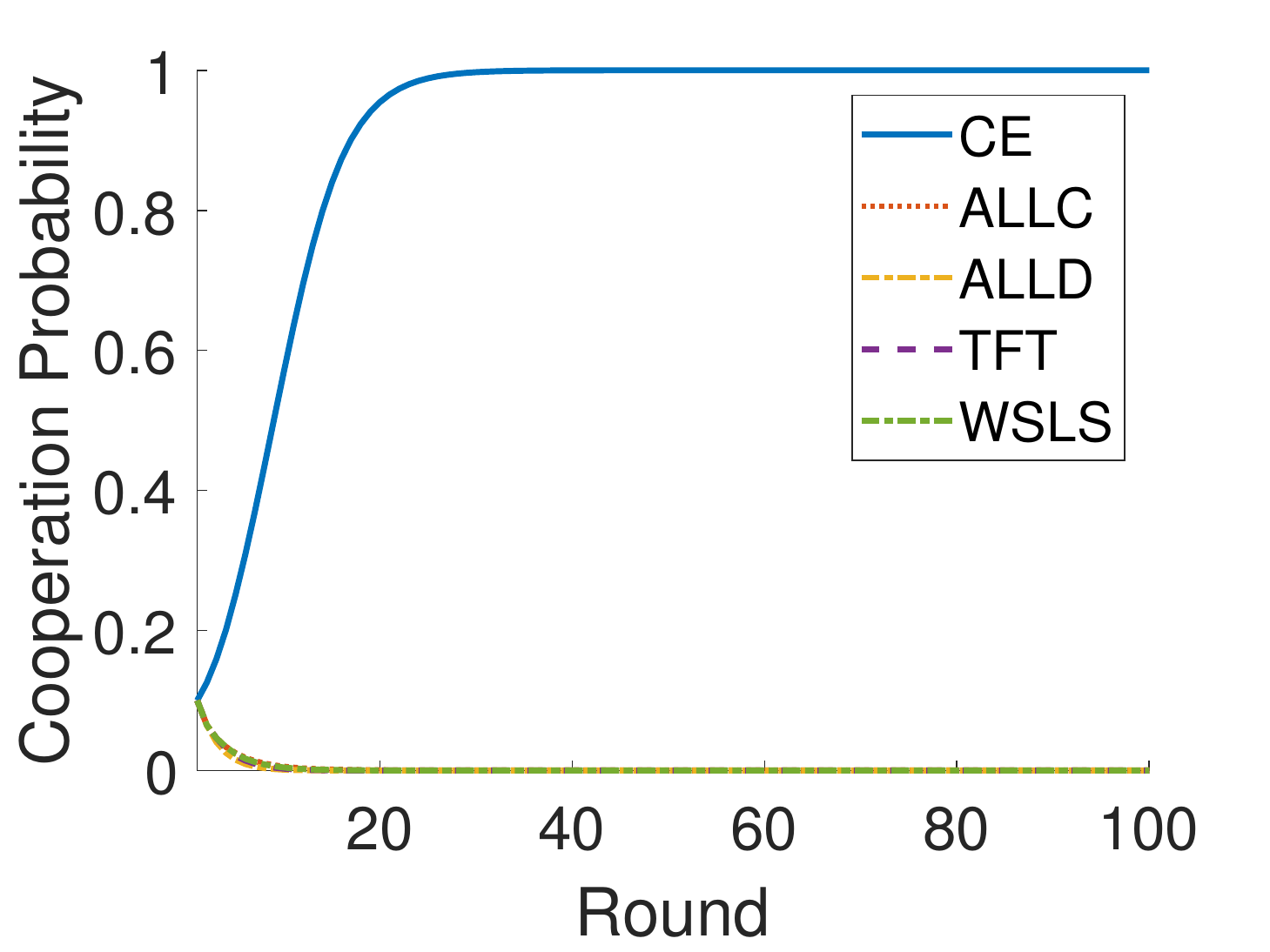}}
\subfigure[$q^0=0.40$.]{
\label{fig:ex1_4}
\includegraphics[width=0.23\textwidth]{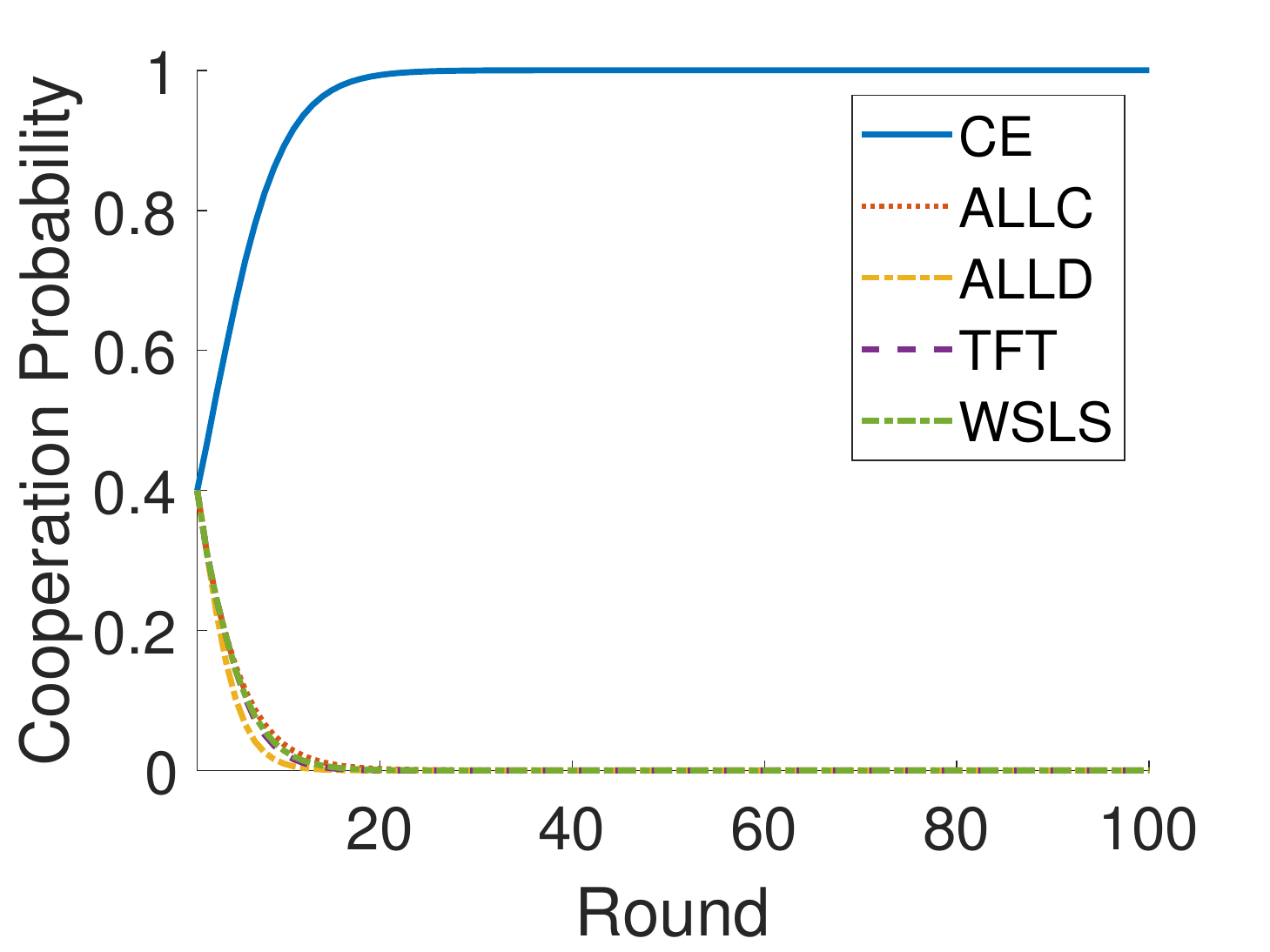}}
\subfigure[$q^0=0.60$.]{
\label{fig:ex1_6}
\includegraphics[width=0.23\textwidth]{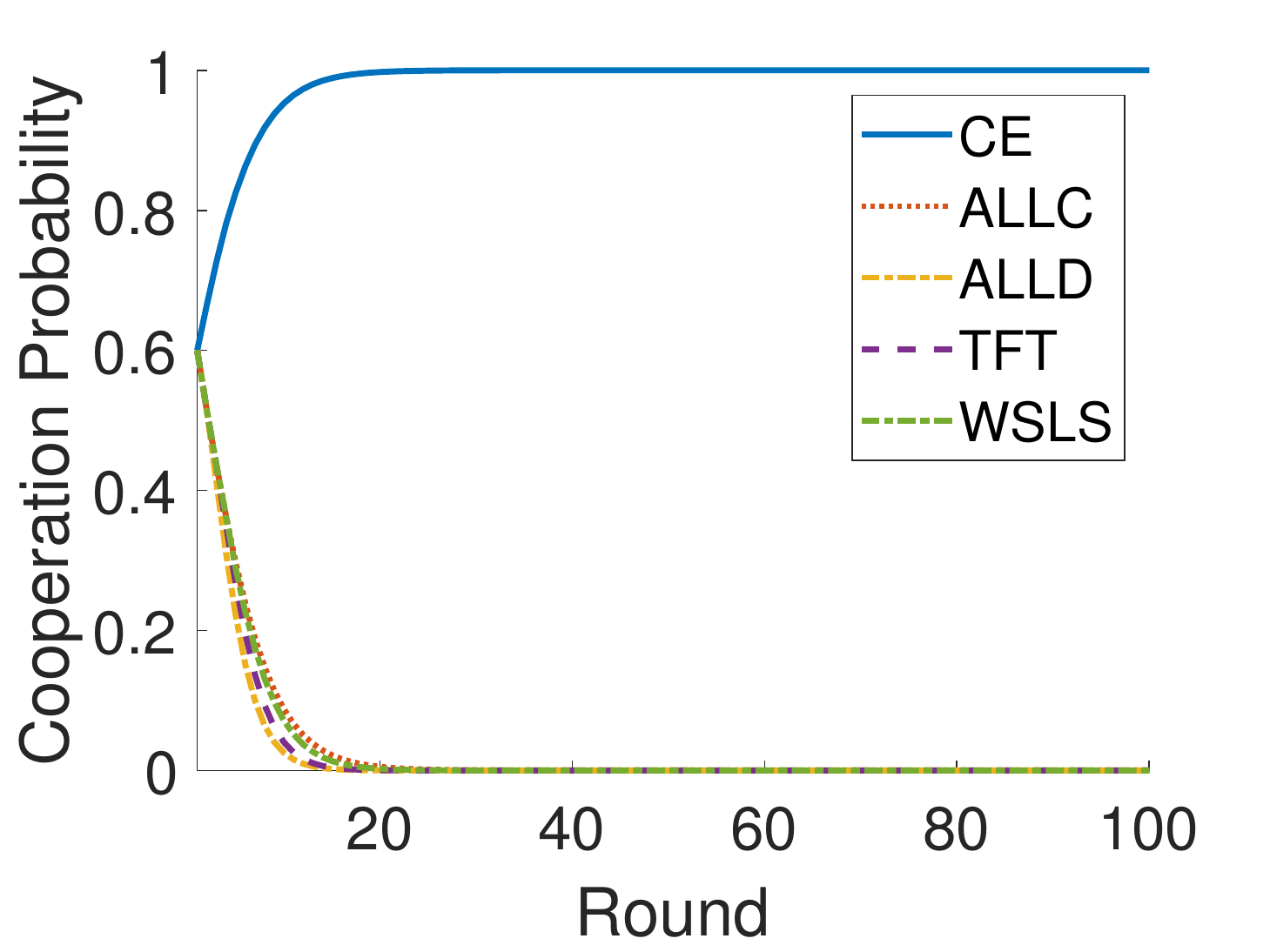}}
\subfigure[$q^0=0.90$.]{
\label{fig:ex1_9}
\includegraphics[width=0.23\textwidth]{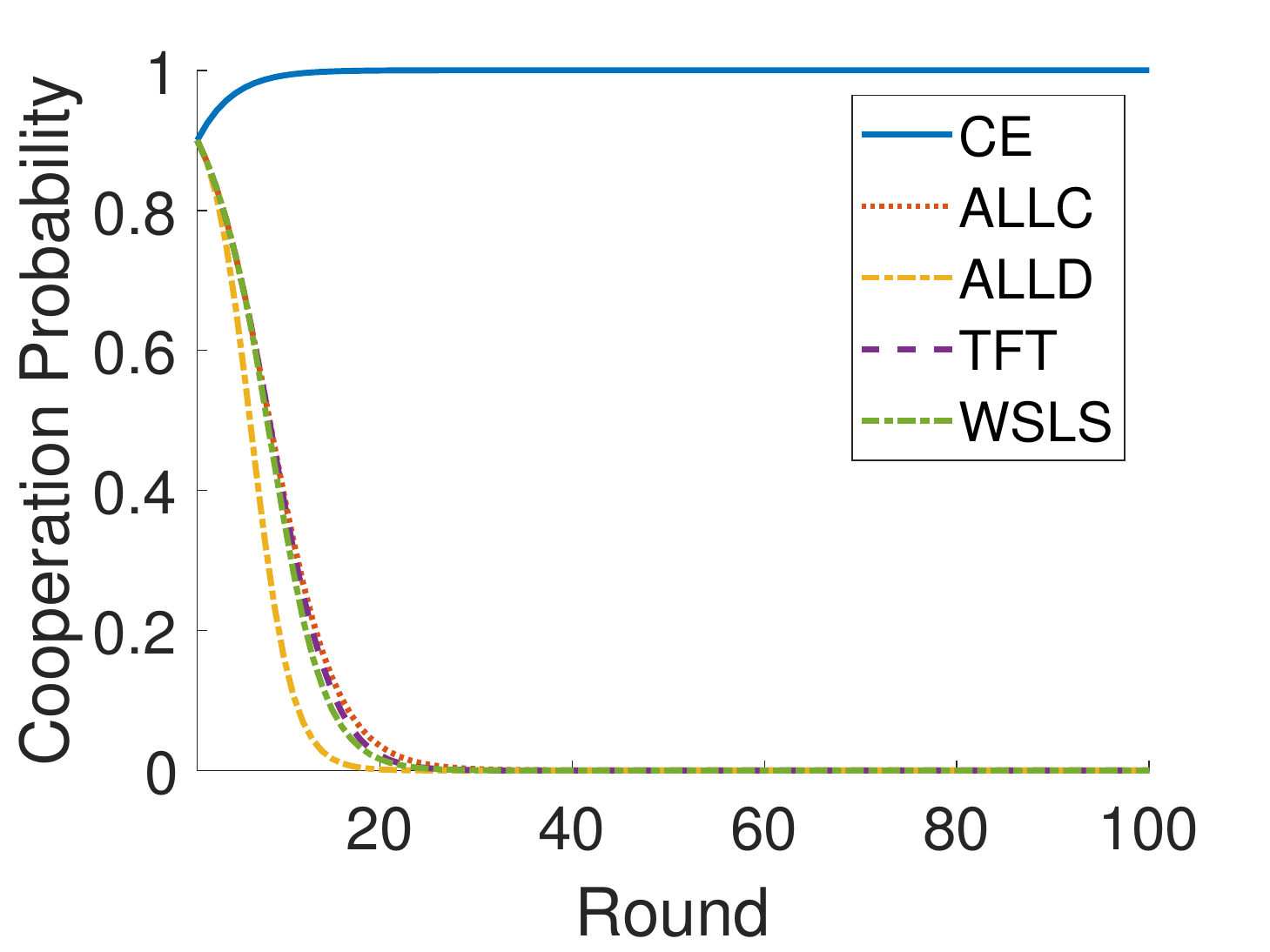}}
\caption{Cooperation probability dynamics of the evolutionary device given different strategies adopted by the server.}
\label{fig:ex1}
\end{figure}

\subsection{Fairness of the CE strategy}
Next, we explore whether the CE strategy is fair for both the server and the devices. We compare their utilities at the stable state in five cases where the server adopts different strategies. Specifically, we set the initial cooperation probability of a device as $q^0=0.4$ in this experiment and present the experimental results in Fig. \ref{fig:ex2}. It is worth noting that since the utility of the server and that of the device are different in values according to the definitions in \eqref{eq:device_utility} and \eqref{eq:server_utility}, we utilize a metric termed \textit{relative utility}, which is calculated by the ratio of the actual utility to the utility at the all-cooperation state, to study the fairness of each strategy. 

According to Fig. \ref{fig:ex2}, one can find that only the proposed CE strategy can achieve almost the same relative utility level for both the server and the device, which approximately equals 1, indicating that both of them obtain the stable utility with the value equivalent to the utility when all cooperate, i.e., $u_s^1$ and $u_i^1$. This clearly demonstrates the fairness of the CE strategy in incentivizing the full cooperation of all devices. For other cases, one can find that the ALLC strategy makes the server suffer from a severe loss since the evolutionary device can strategically exploit her friendliness and behave defectively to obtain a higher utility. The ALLD and TFT strategies lead to similar results for them where the server gains slightly less than the device. This is because the server cannot be fully exploited with ALLD and TFT strategies but the device in these cases will not be driven to cooperate, and thus both of them obtain less profit compared to the situation where the server adopts the CE strategy. The WSLS strategy also makes the server acquire less but performs better than the case of ALLC.

\begin{figure}[htbp]
\centering
\includegraphics[width=0.35\textwidth]{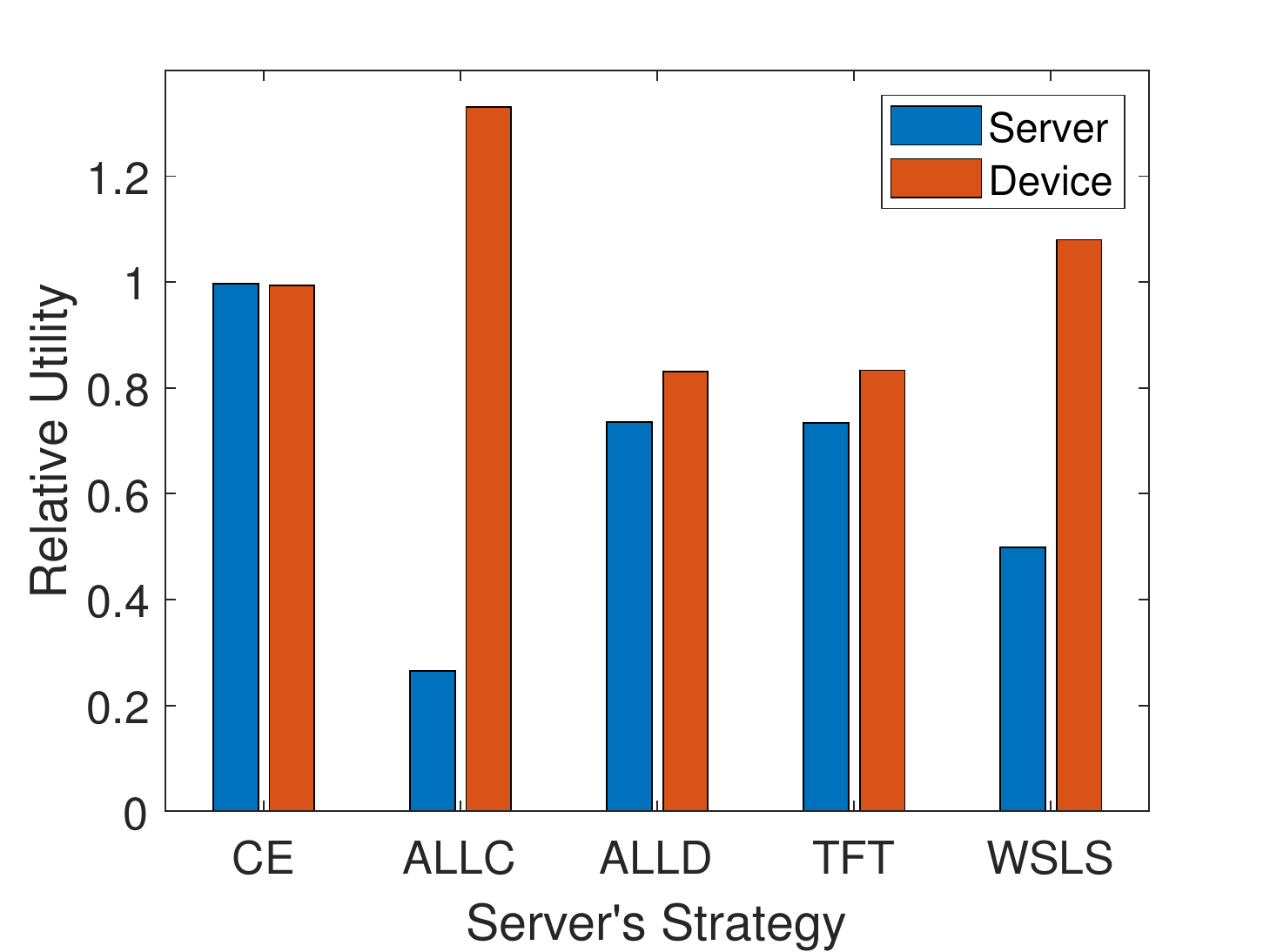}
\caption{Stable relative utilities of the server and the device given different strategies adopted by the server.}
\label{fig:ex2}
\end{figure}

Knowing that the CE strategy can lead to full cooperation of any evolutionary device and achieve almost the same level of stable utilities for both sides, we continue to investigate the dynamics of utility changing with time. In Figs. \ref{fig:ex3_stable} and \ref{fig:ex3_dynamic}, we first plot both utilities at the stable state  with four initial device cooperation probabilities, and then depict the dynamic change of the utilities in each round with each reflecting one case of $q^0$. It can be observed that $q^0$ brings no difference to the stable utilities as shown in the bar graph, while the dynamics of utilities varies according to the device's initial cooperation probability. More specifically, with the increase of $q^0$, the utilities of both sides can converge faster. In other words, the more cooperative the devices at the beginning, the quicker they can reach the stable state, which is coincident with the results of cooperation probability's evolution presented in Fig. \ref{fig:ex1}.

\begin{figure}[htbp]
\centering
\includegraphics[width=0.35\textwidth]{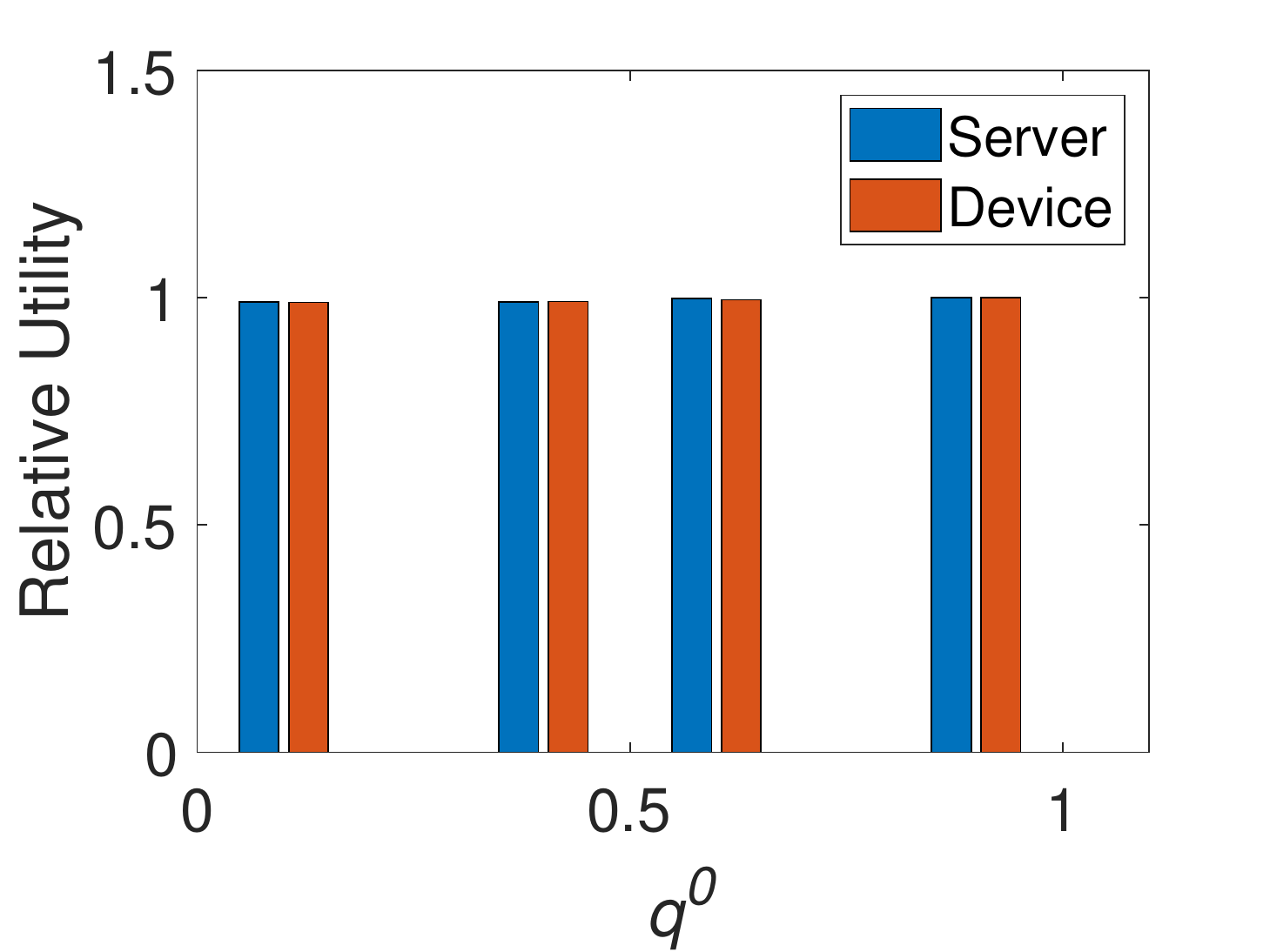}
\caption{Relative utilities of the server and the device at the stable state.}
\label{fig:ex3_stable}
\end{figure}

\begin{figure}[htbp]
\subfigure[$q^0=0.10$.]{
\includegraphics[width=0.23\textwidth]{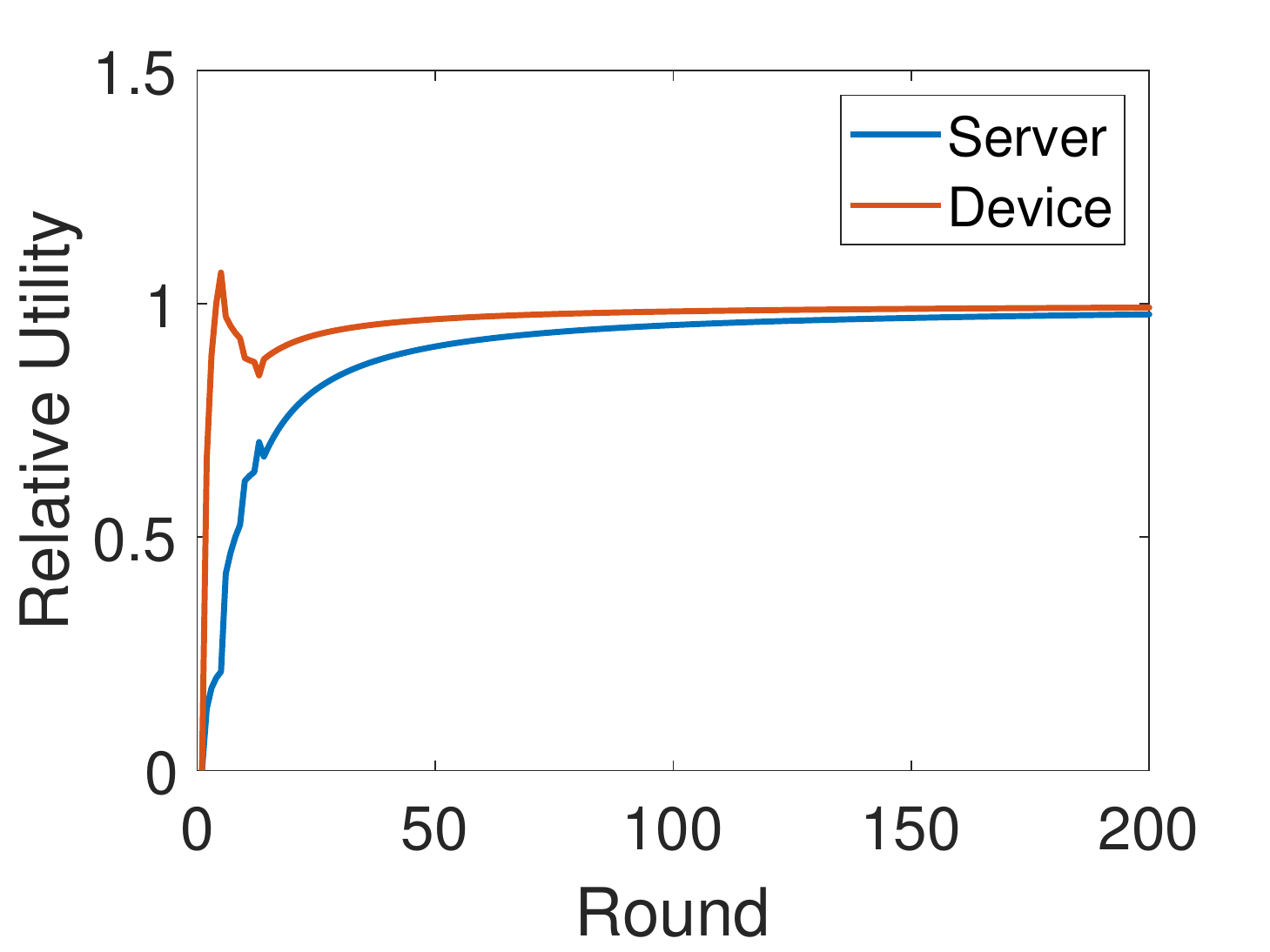}}
\subfigure[$q^0=0.40$.]{
\includegraphics[width=0.23\textwidth]{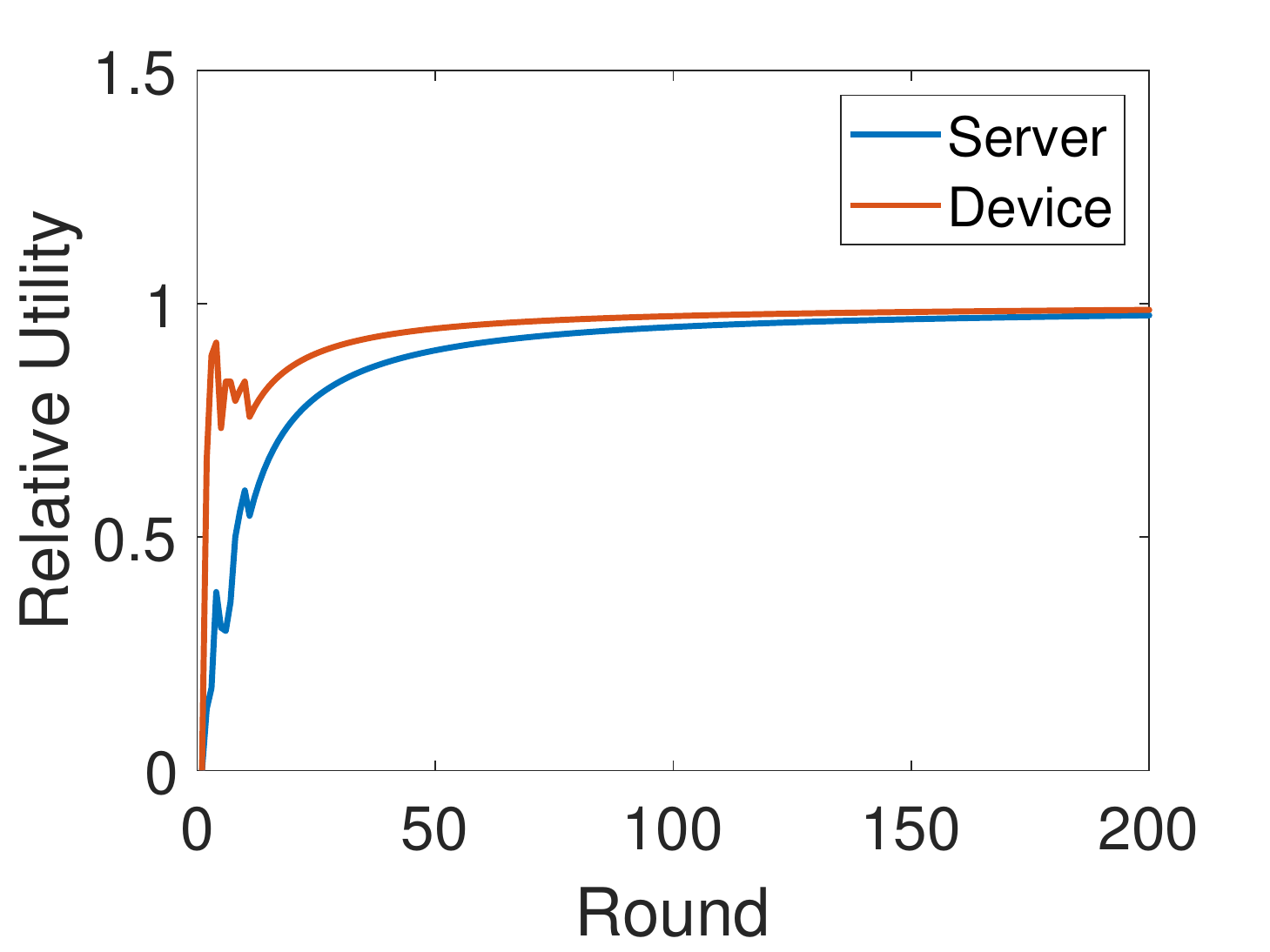}}
\subfigure[$q^0=0.60$.]{
\includegraphics[width=0.23\textwidth]{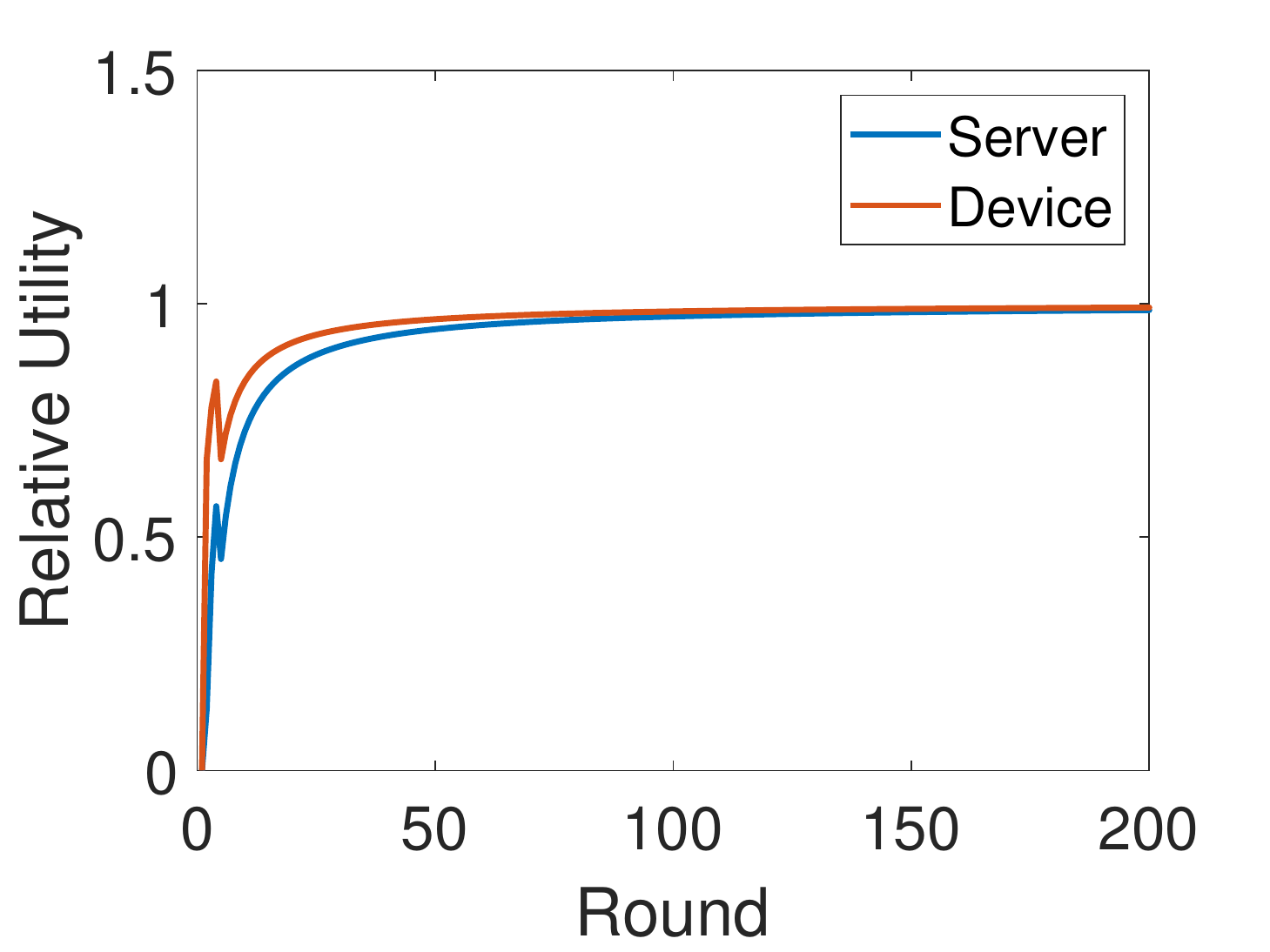}}
\subfigure[$q^0=0.90$.]{
\includegraphics[width=0.23\textwidth]{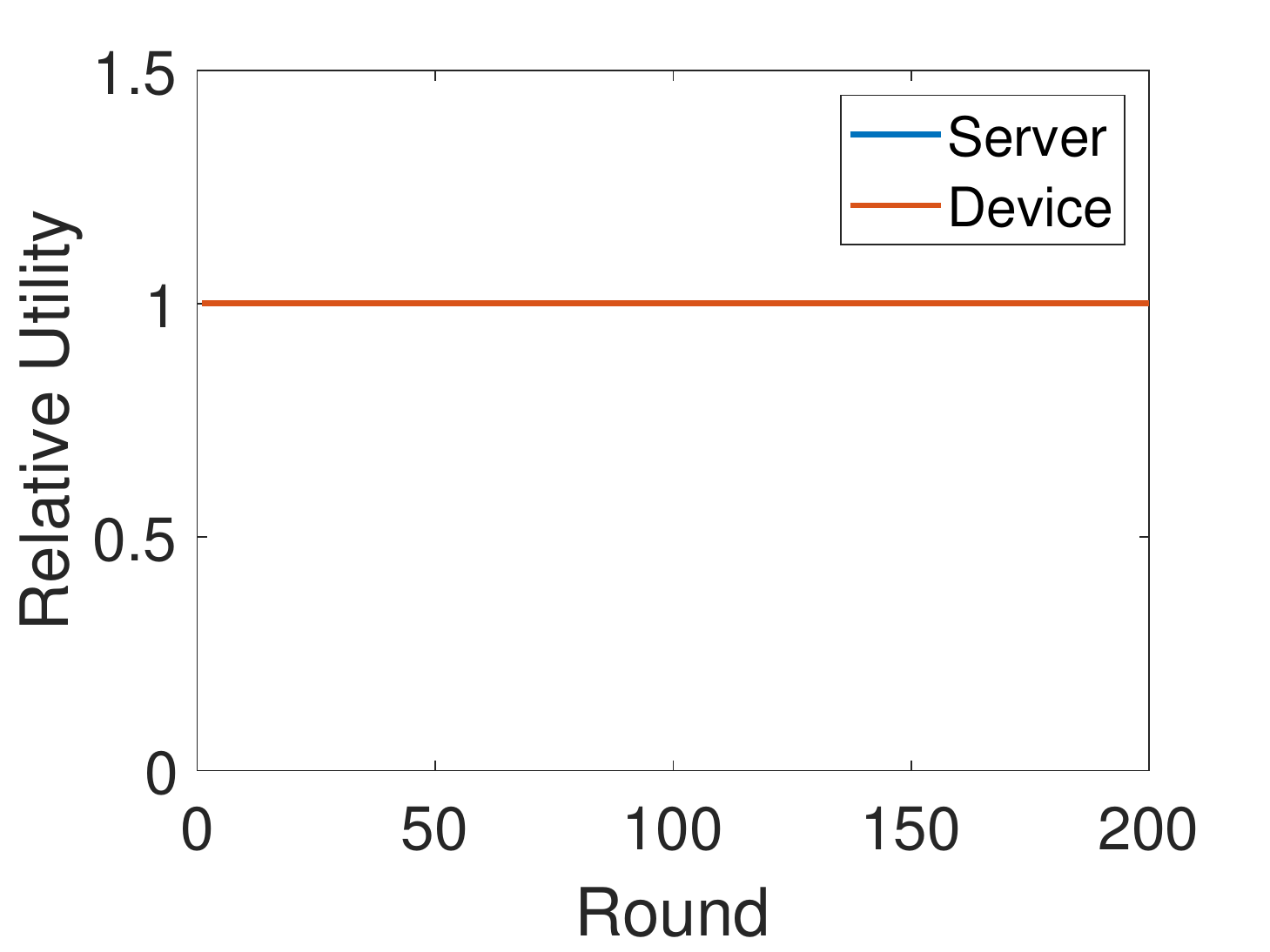}}
\centering
\caption{Relative utilities of the server and the device in dynamic.}
\label{fig:ex3_dynamic}
\end{figure}

Further, we study the dynamics of utilities with the server adopting four other classical strategies and report the experimental results in Fig. \ref{fig:ex4}. One can find that four classical strategies bring different evolution utility paths, especially at the beginning, but all of them converge to the stable result in which the server obtains less utility than the device, which cannot meet the server's expectation. 

\begin{figure}[h]
\centering
\subfigure[ALLC.]{
\label{fig:ex4_allc}
\includegraphics[width=0.23\textwidth]{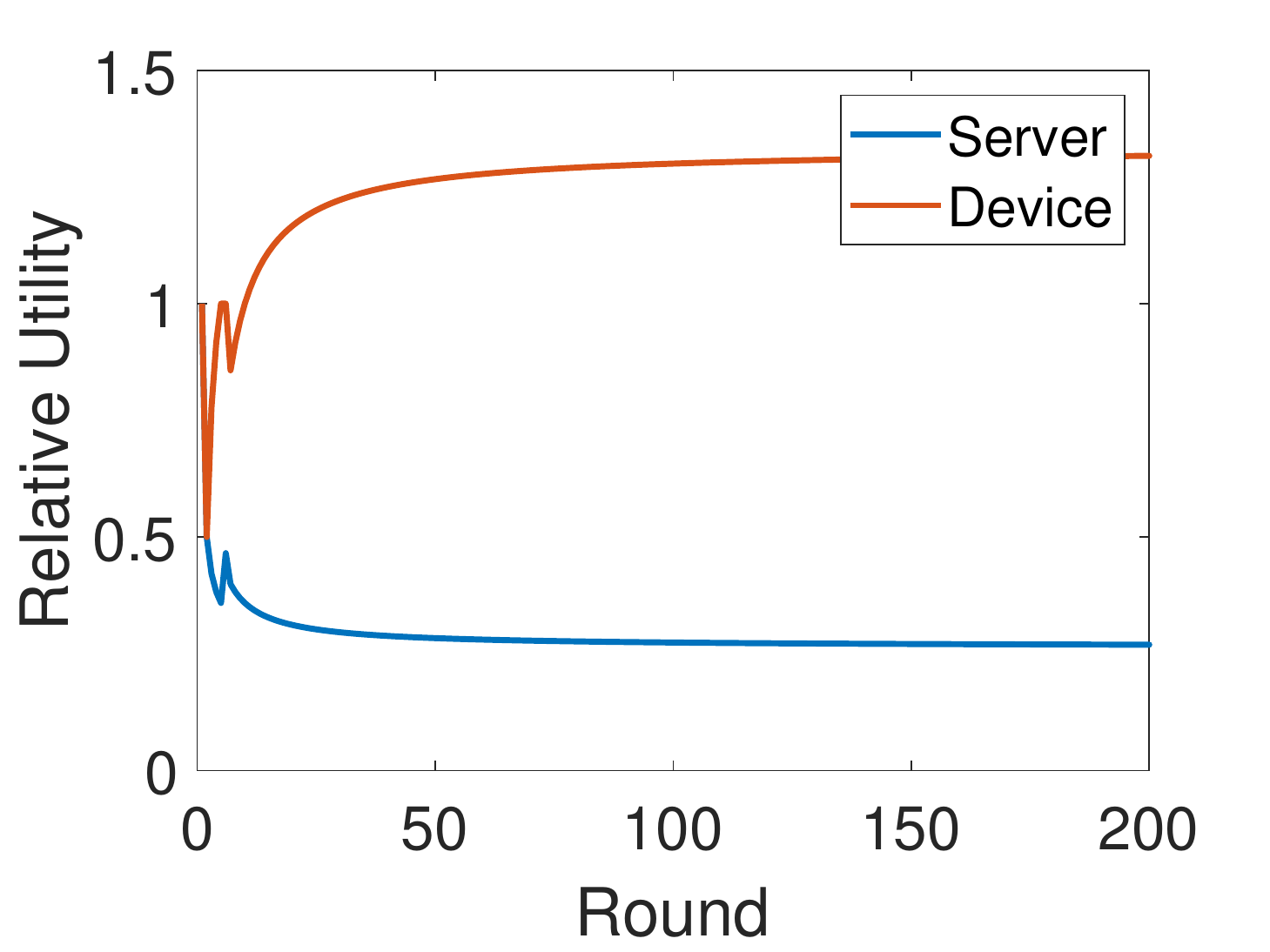}}
\subfigure[ALLD.]{
\label{fig:ex4_alld}
\includegraphics[width=0.23\textwidth]{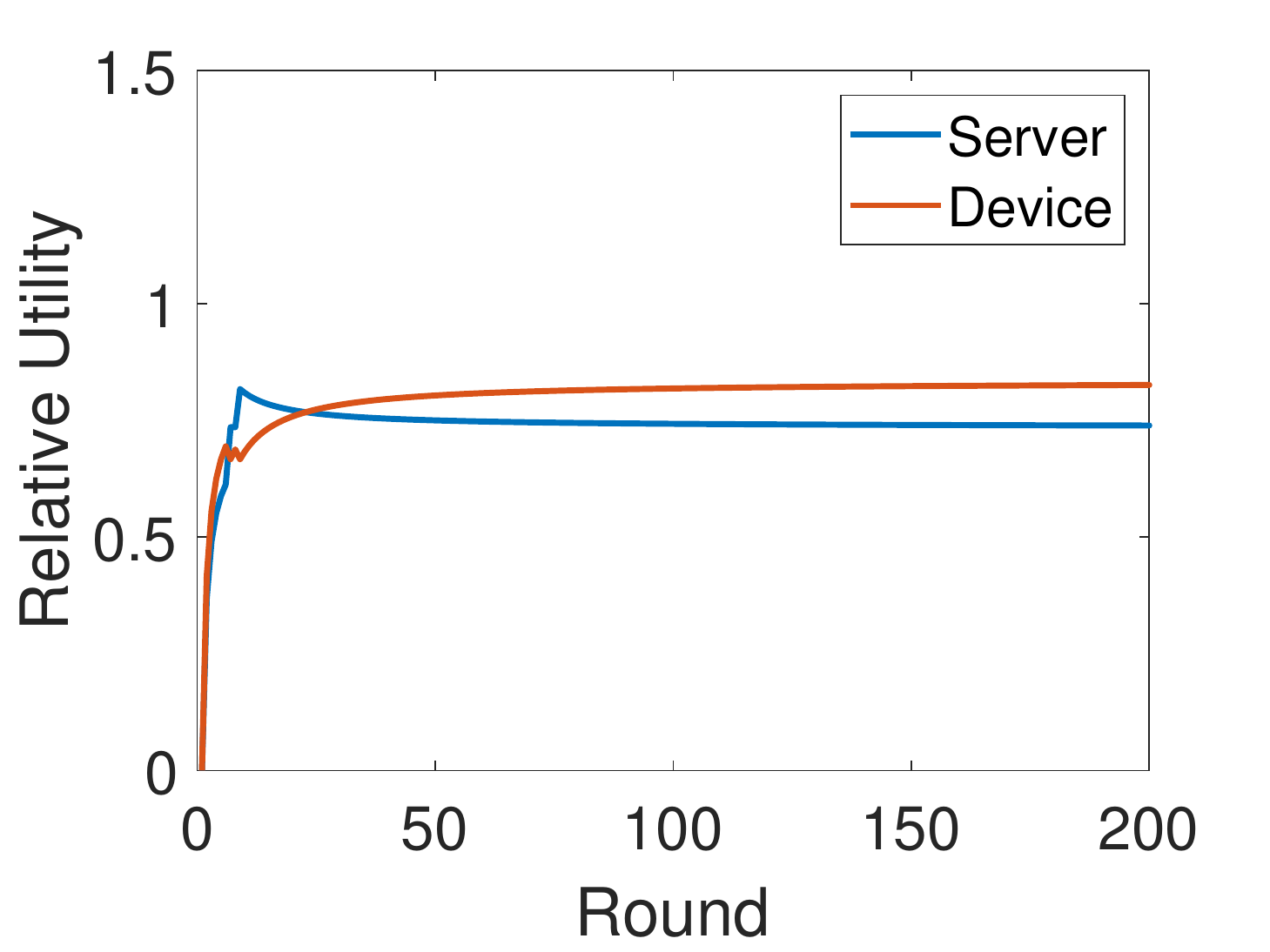}}
\subfigure[TFT.]{
\label{fig:ex4_tft}
\includegraphics[width=0.23\textwidth]{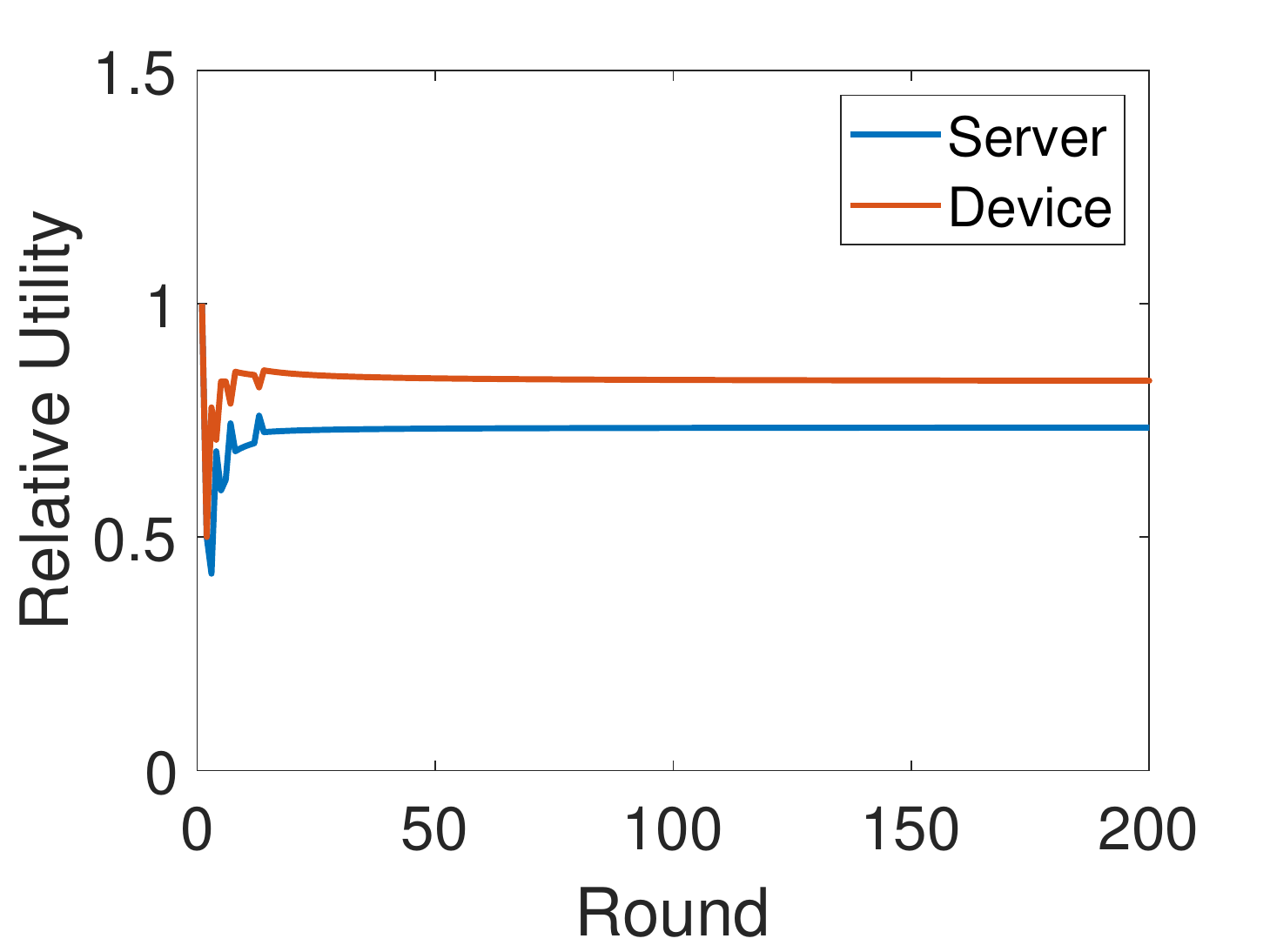}}
\subfigure[WSLS.]{
\label{fig:ex4_wsls}
\includegraphics[width=0.23\textwidth]{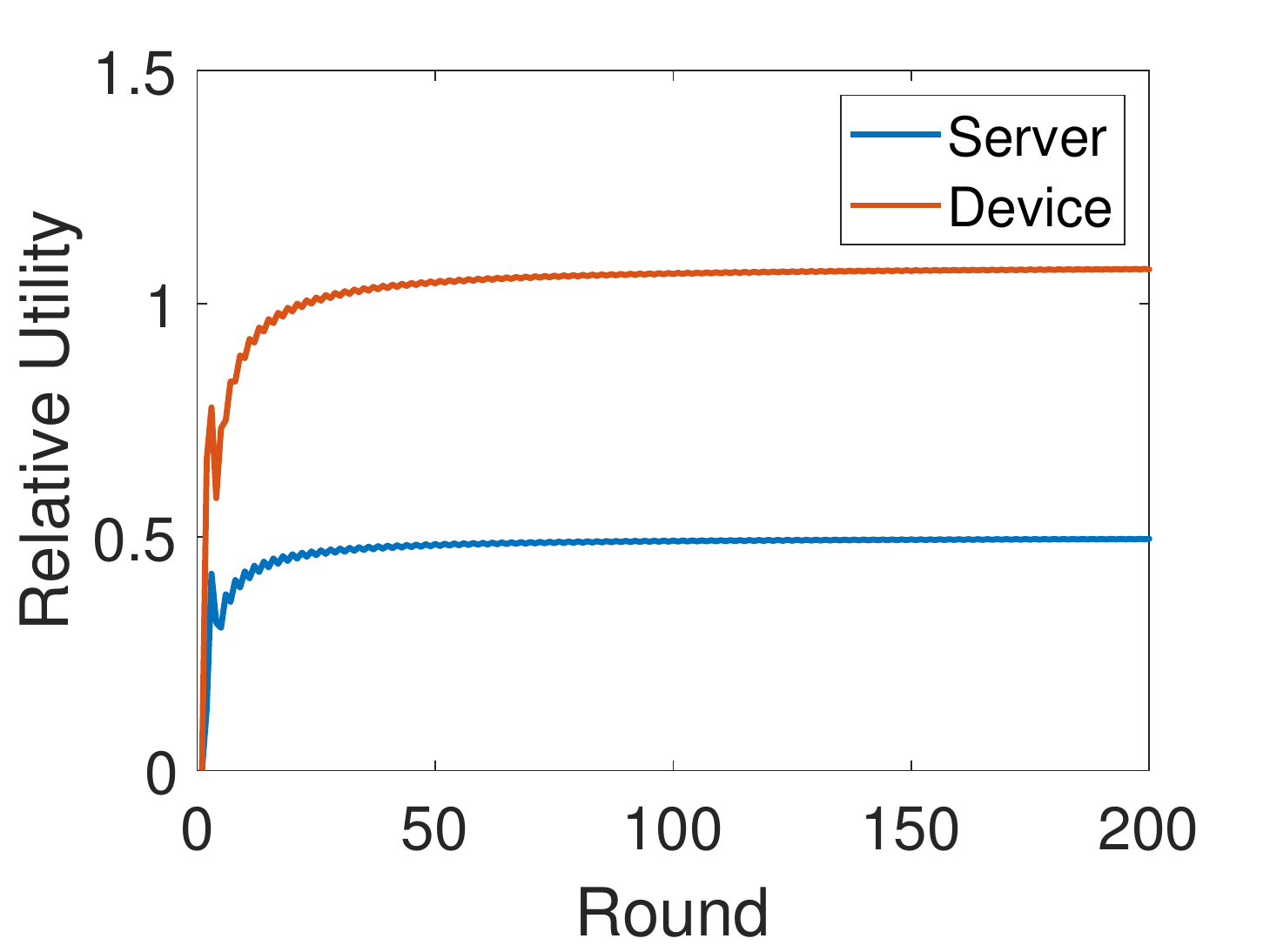}}
\caption{Dynamic relative utilities of the server and the device given different strategies adopted by the server.}
\label{fig:ex4}
\end{figure}

\subsection{Impacts of the extortion factor}

As can be observed in Section \ref{sec:extortion}, the extortion factor $\chi$ in \eqref{eq:extortion_payoff} plays an important role in affecting the degree of utility difference between the server and all devices. To uncover the impact of $\chi$ on the FEL game, we investigate the changing trend of the cooperation probability from any device and the corresponding utility evolution dynamics with different extortion factor in this section, where the initial cooperation probability of the device is set as $q^0=0.5$. 
Detailed experimental results are respectively reported in Figs. \ref{fig:cooperation_chi} and \ref{fig:utility_chi}.

According to Fig. \ref{fig:cooperation_chi}, we can observe that the higher the extortion factor, the longer time is needed for the device becoming fully cooperative. Taking the case of $\chi=1$ as an example, the convergence round of realizing $q^t=1$ is about 10; while for $\chi=4$, the cooperation probability of the device converges to 1 after 50 rounds. 
This phenomenon suggests that even though the server can relatively dominate in the FEL game using the CE strategy, it is not a wise choice for her to enforce severely imbalance expected utilities since the time consumption for eliciting the cooperation from devices can be large. 

With respect to the impact of $\chi$ on the utilities of the server and the device, we can have some clues from Fig. \ref{fig:utility_chi}. Although the specific evolution paths of the instant utilities are different with varying $\chi$, the stable results are the same where each player obtains the utility of mutual cooperation. This outcome implies that the extortion factor in the CE strategy has few impact on the utilities that each player can obtain at the stable state. The underlying reason is that the power CE strategy can drive the device to fully collaborate given any $\chi$, which leads to mutual cooperation and thus the same level of relative utilities for all players. In fact, this consequence is also complying with the fairness feature of the CE strategy as presented earlier.

\begin{figure}[htbp]
\centering
\includegraphics[width=0.35\textwidth]{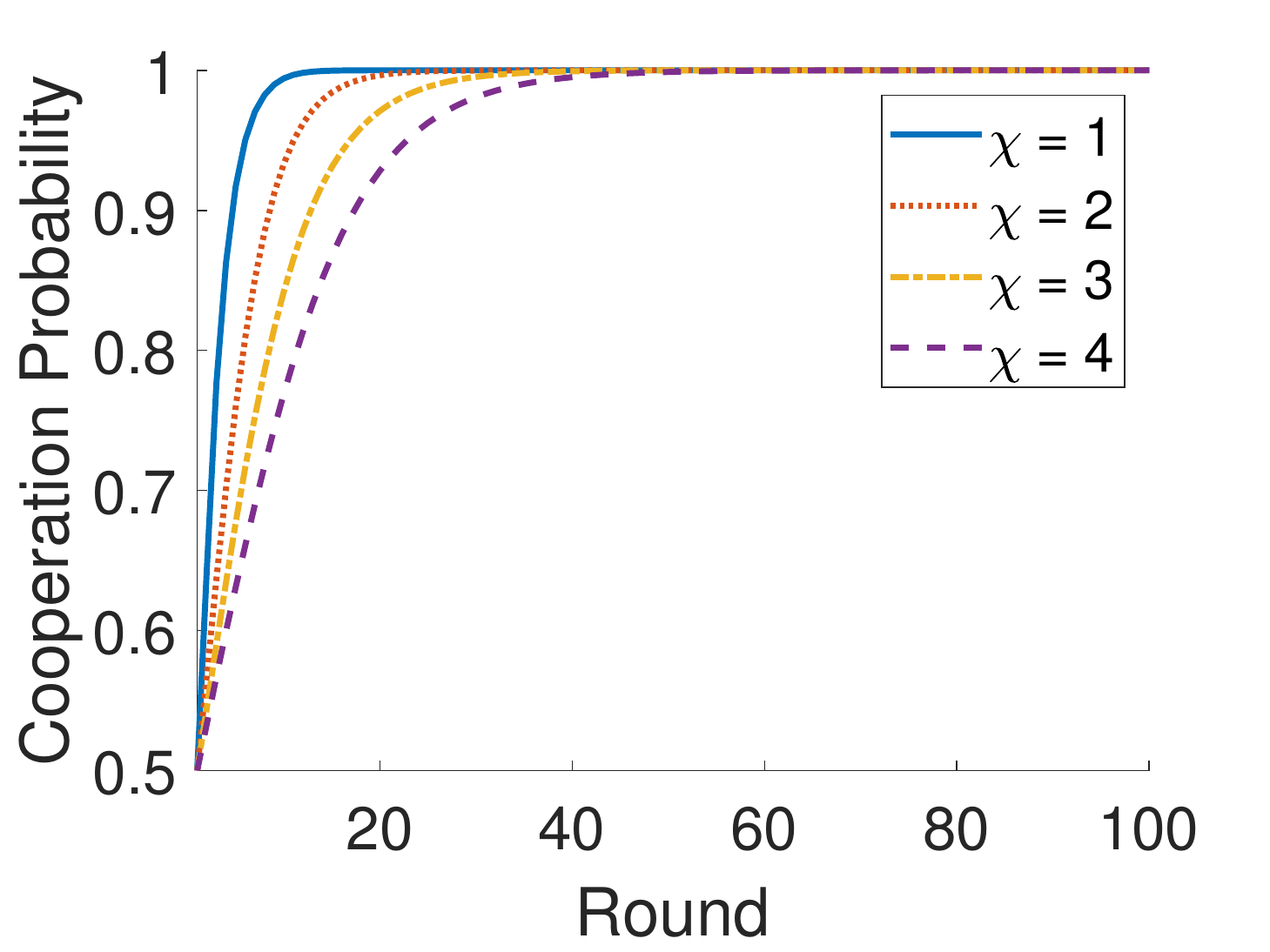}
\caption{Cooperation probability dynamics of the evolutionary device for different $\chi$.}
\label{fig:cooperation_chi}
\end{figure}

\begin{figure}[htbp]
\subfigure[$\chi=1$.]{
\includegraphics[width=0.23\textwidth]{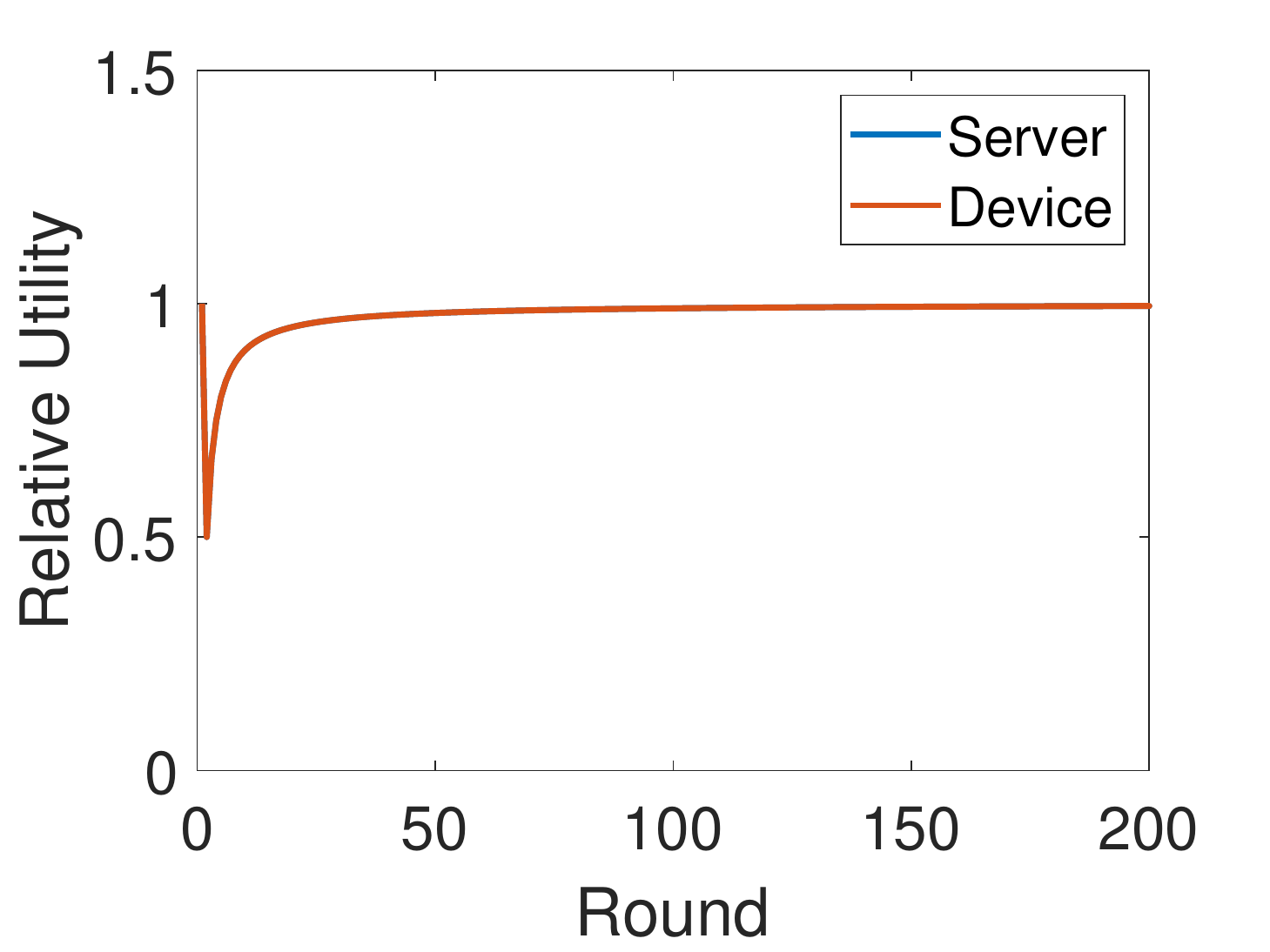}}
\subfigure[$\chi=2$.]{
\includegraphics[width=0.23\textwidth]{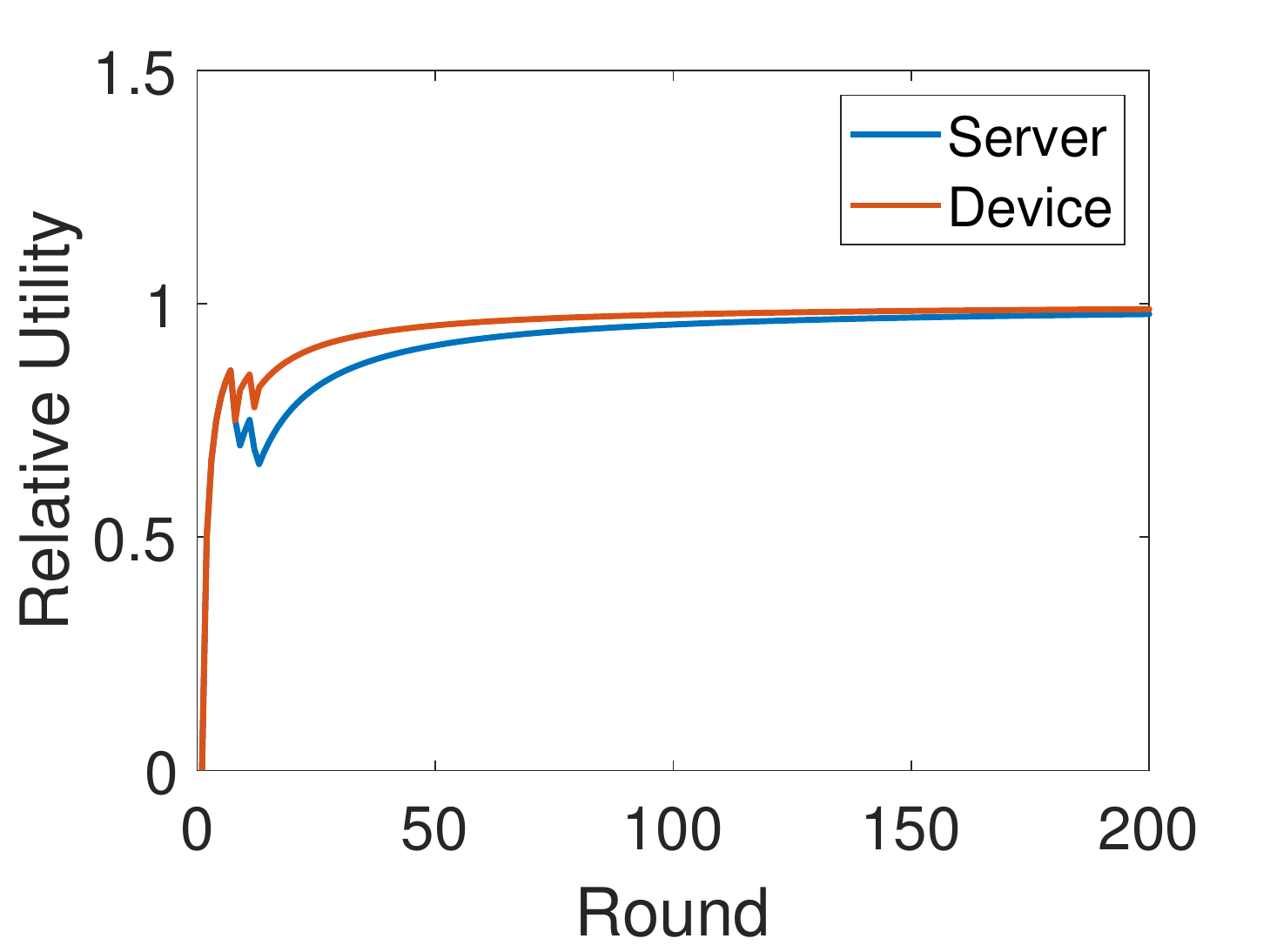}}
\subfigure[$\chi=3$.]{
\includegraphics[width=0.23\textwidth]{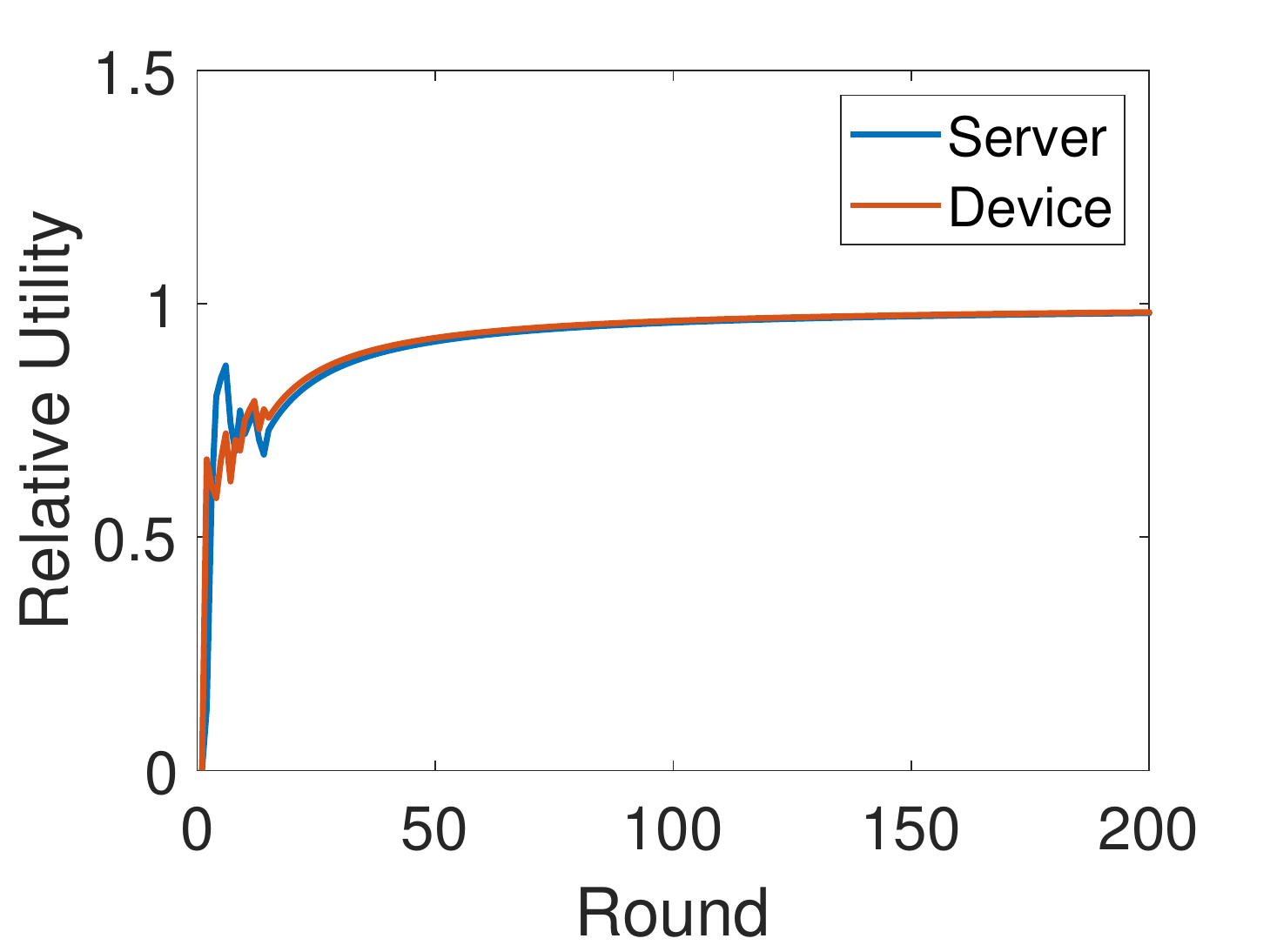}}
\subfigure[$\chi=4$.]{
\includegraphics[width=0.23\textwidth]{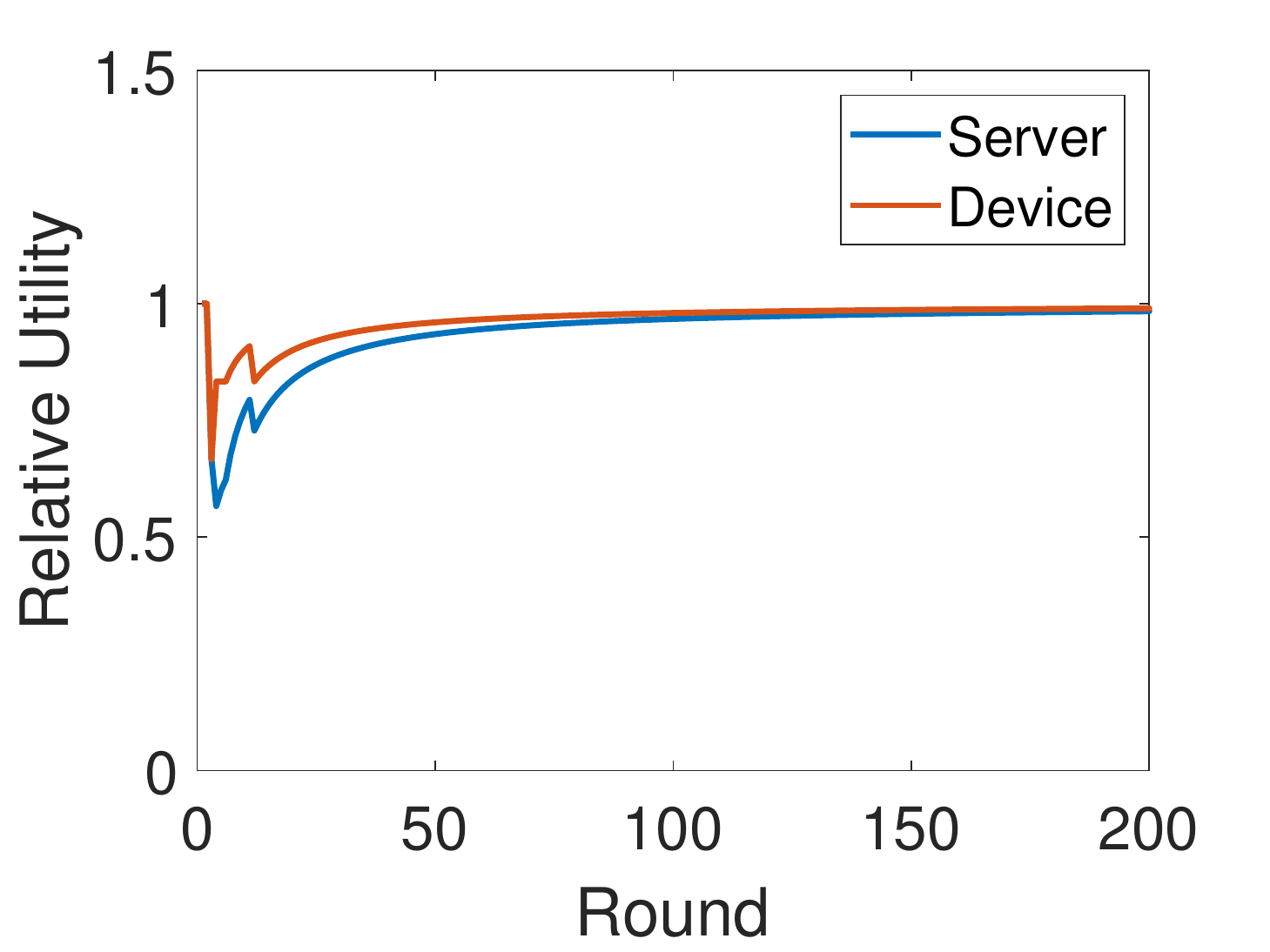}}
\centering
\caption{Relative utilities of the server and the device in dynamic with varying $\chi$.}
\label{fig:utility_chi}
\end{figure}

\section{Conclusion}\label{sec:conclusion}

In this paper, we investigate the problem of optimizing the FEL system performance via eliminating the selfish device behaviors. Specifically, we model the interactions between the edge server and the devices as a multi-player simultaneous game, based on which we derive a CE strategy to collectively control the proportional relationship between the utility of the server and that of the devices. Based on this CE strategy, 
the server can efficiently enforce full contribution of all devices without concerning about her utility, 
which is both theoretically analyzed and experimentally evaluated. Essentially, the proposed CE strategy is impartial for both the adopter and the  opponents, indicating its liveness to maintain the stability of the FEL systems.

In the future, we plan to examine the efficiency and scalability of the proposed game-theoretic scheme in playing against selfish devices in FEL. 
Besides, we will explore more intelligent  solutions about countering other malicious behaviors of devices in FEL, where dynamically joining and leaving the learning process will be discussed to describe more realistic scenarios.

\bibliographystyle{IEEEtran}
\bibliography{reference}

\begin{thebibliography}{10}
\providecommand{\url}[1]{#1}
\csname url@samestyle\endcsname
\providecommand{\newblock}{\relax}
\providecommand{\bibinfo}[2]{#2}
\providecommand{\BIBentrySTDinterwordspacing}{\spaceskip=0pt\relax}
\providecommand{\BIBentryALTinterwordstretchfactor}{4}
\providecommand{\BIBentryALTinterwordspacing}{\spaceskip=\fontdimen2\font plus
\BIBentryALTinterwordstretchfactor\fontdimen3\font minus
  \fontdimen4\font\relax}
\providecommand{\BIBforeignlanguage}[2]{{%
\expandafter\ifx\csname l@#1\endcsname\relax
\typeout{** WARNING: IEEEtran.bst: No hyphenation pattern has been}%
\typeout{** loaded for the language `#1'. Using the pattern for}%
\typeout{** the default language instead.}%
\else
\language=\csname l@#1\endcsname
\fi
#2}}
\providecommand{\BIBdecl}{\relax}
\BIBdecl

\bibitem{xiao2019edge}
Y.~Xiao, Y.~Jia, C.~Liu, X.~Cheng, J.~Yu, and W.~Lv, ``Edge computing security:
  State of the art and challenges,'' \emph{Proceedings of the IEEE}, vol. 107,
  no.~8, pp. 1608--1631, 2019.

\bibitem{market}
``Edge computing market,'' \url{https://www.marketsandmarkets.com
  /Market-Reports/edge-computing-market-133384090.html}, accessed: 2020-07-30.

\bibitem{khan2020federated}
L.~U. Khan, S.~R. Pandey, N.~H. Tran, W.~Saad, Z.~Han, M.~N. Nguyen, and C.~S.
  Hong, ``Federated learning for edge networks: Resource optimization and
  incentive mechanism,'' \emph{IEEE Communications Magazine}, vol.~58, no.~10,
  pp. 88--93, 2020.

\bibitem{lim2021towards}
W.~Y.~B. Lim, J.~Huang, Z.~Xiong, J.~Kang, D.~Niyato, X.-S. Hua, C.~Leung, and
  C.~Miao, ``Towards federated learning in uav-enabled internet of vehicles: A
  multi-dimensional contract-matching approach,'' \emph{IEEE Transactions on
  Intelligent Transportation Systems}, 2021.

\bibitem{liu2020federated}
Y.~Liu, X.~Yuan, Z.~Xiong, J.~Kang, X.~Wang, and D.~Niyato, ``Federated
  learning for 6g communications: Challenges, methods, and future directions,''
  \emph{China Communications}, vol.~17, no.~9, pp. 105--118, 2020.

\bibitem{mills2020communication}
J.~{Mills}, J.~{Hu}, and G.~{Min}, ``Communication-efficient federated learning
  for wireless edge intelligence in iot,'' \emph{IEEE Internet of Things
  Journal}, vol.~7, no.~7, pp. 5986--5994, 2020.

\bibitem{jiang2019model}
Y.~Jiang, S.~Wang, B.~J. Ko, W.-H. Lee, and L.~Tassiulas, ``Model pruning
  enables efficient federated learning on edge devices,'' \emph{arXiv preprint
  arXiv:1909.12326}, 2019.

\bibitem{wang2019adaptive}
S.~Wang, T.~Tuor, T.~Salonidis, K.~K. Leung, C.~Makaya, T.~He, and K.~Chan,
  ``Adaptive federated learning in resource constrained edge computing
  systems,'' \emph{IEEE Journal on Selected Areas in Communications}, vol.~37,
  no.~6, pp. 1205--1221, 2019.

\bibitem{zhu2019broadband}
G.~Zhu, Y.~Wang, and K.~Huang, ``Broadband analog aggregation for low-latency
  federated edge learning,'' \emph{IEEE Transactions on Wireless
  Communications}, vol.~19, no.~1, pp. 491--506, 2019.

\bibitem{amiri2020machine}
M.~M. Amiri and D.~G{\"u}nd{\"u}z, ``Machine learning at the wireless edge:
  Distributed stochastic gradient descent over-the-air,'' \emph{IEEE
  Transactions on Signal Processing}, vol.~68, pp. 2155--2169, 2020.

\bibitem{yang2020federated}
K.~Yang, T.~Jiang, Y.~Shi, and Z.~Ding, ``Federated learning via over-the-air
  computation,'' \emph{IEEE Transactions on Wireless Communications}, vol.~19,
  no.~3, pp. 2022--2035, 2020.

\bibitem{ahn2019wireless}
J.-H. Ahn, O.~Simeone, and J.~Kang, ``Wireless federated distillation for
  distributed edge learning with heterogeneous data,'' in \emph{2019 IEEE 30th
  Annual International Symposium on Personal, Indoor and Mobile Radio
  Communications (PIMRC)}.\hskip 1em plus 0.5em minus 0.4em\relax IEEE, 2019,
  pp. 1--6.

\bibitem{tran2019federated}
N.~H. Tran, W.~Bao, A.~Zomaya, N.~M. NH, and C.~S. Hong, ``Federated learning
  over wireless networks: Optimization model design and analysis,'' in
  \emph{2019 IEEE Conference on Computer Communications (INFOCOM)}.\hskip 1em
  plus 0.5em minus 0.4em\relax IEEE, 2019, pp. 1387--1395.

\bibitem{xu2019elfish}
Z.~Xu, Z.~Yang, J.~Xiong, J.~Yang, and X.~Chen, ``Elfish: Resource-aware
  federated learning on heterogeneous edge devices,'' \emph{arXiv preprint
  arXiv:1912.01684}, 2019.

\bibitem{prakash2020coded}
S.~Prakash, S.~Dhakal, M.~R. Akdeniz, Y.~Yona, S.~Talwar, S.~Avestimehr, and
  N.~Himayat, ``Coded computing for low-latency federated learning over
  wireless edge networks,'' \emph{IEEE Journal on Selected Areas in
  Communications}, vol.~39, no.~1, pp. 233--250, 2020.

\bibitem{abad2020hierarchical}
M.~S.~H. Abad, E.~Ozfatura, D.~Gunduz, and O.~Ercetin, ``Hierarchical federated
  learning across heterogeneous cellular networks,'' in \emph{2020 IEEE
  International Conference on Acoustics, Speech and Signal Processing
  (ICASSP)}.\hskip 1em plus 0.5em minus 0.4em\relax IEEE, 2020, pp. 8866--8870.

\bibitem{zeng2020energy}
Q.~Zeng, Y.~Du, K.~Huang, and K.~K. Leung, ``Energy-efficient radio resource
  allocation for federated edge learning,'' in \emph{2020 IEEE International
  Conference on Communications Workshops (ICC Workshops)}.\hskip 1em plus 0.5em
  minus 0.4em\relax IEEE, 2020, pp. 1--6.

\bibitem{yang2020age}
H.~H. Yang, A.~Arafa, T.~Q. Quek, and H.~V. Poor, ``Age-based scheduling policy
  for federated learning in mobile edge networks,'' in \emph{2020 IEEE
  International Conference on Acoustics, Speech and Signal Processing
  (ICASSP)}.\hskip 1em plus 0.5em minus 0.4em\relax IEEE, 2020, pp. 8743--8747.

\bibitem{yang2019scheduling}
H.~H. Yang, Z.~Liu, T.~Q. Quek, and H.~V. Poor, ``Scheduling policies for
  federated learning in wireless networks,'' \emph{IEEE transactions on
  communications}, vol.~68, no.~1, pp. 317--333, 2019.

\bibitem{amiri2020update}
M.~M. Amiri, D.~G{\"u}nd{\"u}z, S.~R. Kulkarni, and H.~V. Poor, ``Update aware
  device scheduling for federated learning at the wireless edge,'' in
  \emph{2020 IEEE International Symposium on Information Theory (ISIT)}.\hskip
  1em plus 0.5em minus 0.4em\relax IEEE, 2020, pp. 2598--2603.

\bibitem{nishio2019client}
T.~Nishio and R.~Yonetani, ``Client selection for federated learning with
  heterogeneous resources in mobile edge,'' in \emph{2019 IEEE International
  Conference on Communications (ICC)}.\hskip 1em plus 0.5em minus 0.4em\relax
  IEEE, 2019, pp. 1--7.

\bibitem{kang2019incentive}
J.~Kang, Z.~Xiong, D.~Niyato, S.~Xie, and J.~Zhang, ``Incentive mechanism for
  reliable federated learning: A joint optimization approach to combining
  reputation and contract theory,'' \emph{IEEE Internet of Things Journal},
  vol.~6, no.~6, pp. 10\,700--10\,714, 2019.

\bibitem{ye2020federated}
D.~Ye, R.~Yu, M.~Pan, and Z.~Han, ``Federated learning in vehicular edge
  computing: A selective model aggregation approach,'' \emph{IEEE Access},
  vol.~8, pp. 23\,920--23\,935, 2020.

\bibitem{zhan2020learning}
Y.~Zhan, P.~Li, Z.~Qu, D.~Zeng, and S.~Guo, ``A learning-based incentive
  mechanism for federated learning,'' \emph{IEEE Internet of Things Journal},
  vol.~7, no.~7, pp. 6360--6368, 2020.

\bibitem{zhan2020infocom}
Y.~{Zhan} and J.~{Zhang}, ``An incentive mechanism design for efficient edge
  learning by deep reinforcement learning approach,'' in \emph{IEEE INFOCOM
  2020 - IEEE Conference on Computer Communications}, 2020, pp. 2489--2498.

\bibitem{pandey2020crowdsourcing}
S.~R. Pandey, N.~H. Tran, M.~Bennis, Y.~K. Tun, A.~Manzoor, and C.~S. Hong, ``A
  crowdsourcing framework for on-device federated learning,'' \emph{IEEE
  Transactions on Wireless Communications}, vol.~19, no.~5, pp. 3241--3256,
  2020.

\bibitem{le2021incentive}
T.~H.~T. Le, N.~H. Tran, Y.~K. Tun, M.~N. Nguyen, S.~R. Pandey, Z.~Han, and
  C.~S. Hong, ``An incentive mechanism for federated learning in wireless
  cellular network: An auction approach,'' \emph{IEEE Transactions on Wireless
  Communications (Early Access)}, 2021.

\bibitem{nie2018stackelberg}
J.~Nie, J.~Luo, Z.~Xiong, D.~Niyato, and P.~Wang, ``A stackelberg game approach
  toward socially-aware incentive mechanisms for mobile crowdsensing,''
  \emph{IEEE Transactions on Wireless Communications}, vol.~18, no.~1, pp.
  724--738, 2018.

\bibitem{press2012iterated}
W.~H. Press and F.~J. Dyson, ``Iterated prisoner’s dilemma contains
  strategies that dominate any evolutionary opponent,'' \emph{Proceedings of
  the National Academy of Sciences}, vol. 109, no.~26, pp. 10\,409--10\,413,
  2012.

\bibitem{lim2021decentralized}
W.~Y.~B. Lim, J.~S. Ng, Z.~Xiong, J.~Jin, Y.~Zhang, D.~Niyato, C.~Leung, and
  C.~Miao, ``Decentralized edge intelligence: A dynamic resource allocation
  framework for hierarchical federated learning,'' \emph{IEEE Transactions on
  Parallel and Distributed Systems}, vol.~33, no.~3, pp. 536--550, 2021.

\bibitem{ng2020joint}
J.~S. Ng, W.~Y.~B. Lim, H.-N. Dai, Z.~Xiong, J.~Huang, D.~Niyato, X.-S. Hua,
  C.~Leung, and C.~Miao, ``Joint auction-coalition formation framework for
  communication-efficient federated learning in uav-enabled internet of
  vehicles,'' \emph{IEEE Transactions on Intelligent Transportation Systems},
  vol.~22, no.~4, pp. 2326--2344, 2020.

\bibitem{kang2020reliable}
J.~Kang, Z.~Xiong, D.~Niyato, Y.~Zou, Y.~Zhang, and M.~Guizani, ``Reliable
  federated learning for mobile networks,'' \emph{IEEE Wireless
  Communications}, vol.~27, no.~2, pp. 72--80, 2020.

\bibitem{lim2020hierarchical}
W.~Y.~B. Lim, Z.~Xiong, C.~Miao, D.~Niyato, Q.~Yang, C.~Leung, and H.~V. Poor,
  ``Hierarchical incentive mechanism design for federated machine learning in
  mobile networks,'' \emph{IEEE Internet of Things Journal}, vol.~7, no.~10,
  pp. 9575--9588, 2020.

\bibitem{chen2018my}
I.~Chen, F.~D. Johansson, and D.~Sontag, ``Why is my classifier
  discriminatory?'' in \emph{Advances in Neural Information Processing
  Systems}, 2018, pp. 3539--3550.

\bibitem{johnson2018predicting}
M.~Johnson, P.~Anderson, M.~Dras, and M.~Steedman, ``Predicting accuracy on
  large datasets from smaller pilot data,'' in \emph{Proceedings of the 56th
  Annual Meeting of the Association for Computational Linguistics (Volume 2:
  Short Papers)}, 2018, pp. 450--455.

\bibitem{hao2018payoff}
D.~Hao, K.~Li, and T.~Zhou, ``Payoff control in the iterated prisoner's
  dilemma,'' in \emph{Proceedings of the 27th International Joint Conference on
  Artificial Intelligence}, 2018, pp. 296--302.

\bibitem{smith1982evolution}
J.~M. Smith and J.~M.~M. Smith, \emph{Evolution and the Theory of Games}.\hskip
  1em plus 0.5em minus 0.4em\relax Cambridge university press, 1982.

\bibitem{lecun1998gradient}
Y.~LeCun, L.~Bottou, Y.~Bengio, and P.~Haffner, ``Gradient-based learning
  applied to document recognition,'' \emph{Proceedings of the IEEE}, vol.~86,
  no.~11, pp. 2278--2324, 1998.

\end{thebibliography}

\begin{IEEEbiography}[{\includegraphics[width=1in,height=1.25in,clip,keepaspectratio]{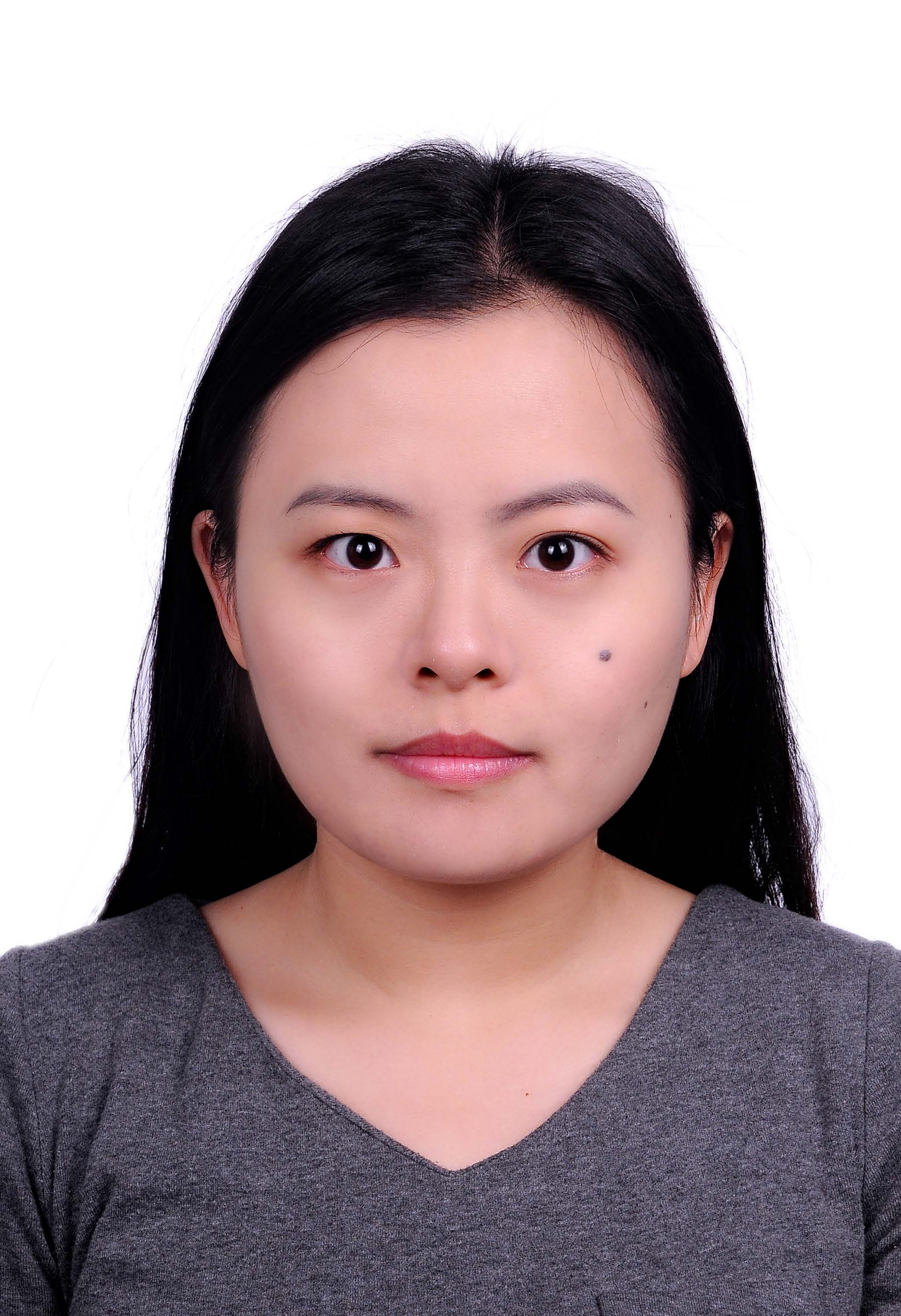}}]{Qin Hu} received her Ph.D. degree in Computer Science from the George Washington University in 2019. She is currently an Assistant Professor with the Department of Computer and Information Science, Indiana University-Purdue University Indianapolis (IUPUI). Her research interests include wireless and mobile security, edge computing, blockchain, and crowdsourcing/crowdsensing.
\end{IEEEbiography}

\begin{IEEEbiography}[{\includegraphics[width=1in,height=1.25in,clip,keepaspectratio]{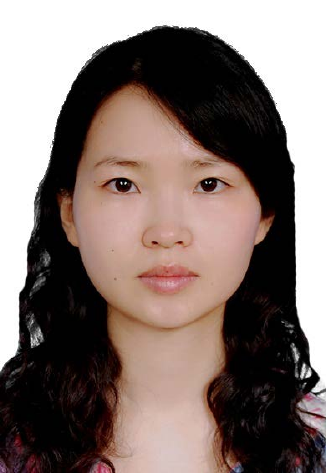}}]{Shengling Wang}
is a full professor in the School
of Artificial Intelligence, Beijing Normal University.
She received her Ph.D. in 2008 from Xian Jiaotong
University. After that, she did her postdoctoral research
in the Department of Computer Science and
Technology, Tsinghua University. Then she worked
as an assistant and associate professor from 2010
to 2013 in the Institute of Computing Technology
of the Chinese Academy of Sciences. Her research
interests include mobile/wireless networks, game theory, crowdsourcing.
\end{IEEEbiography}

\begin{IEEEbiography}[{\includegraphics[width=1in,height=1.25in,clip,keepaspectratio]{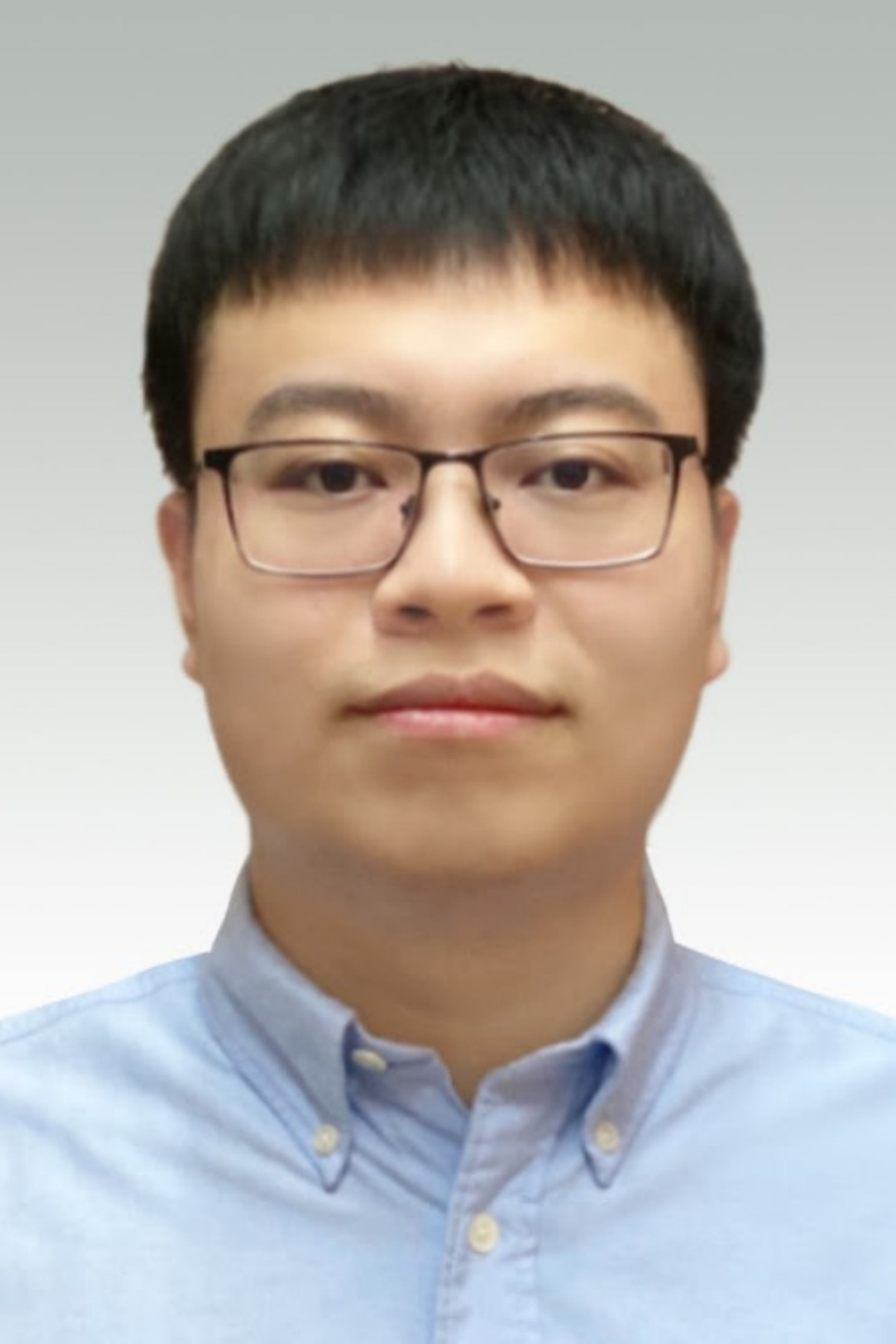}}]{Zehui Xiong} is currently an Assistant Professor in the Pillar of Information Systems Technology and Design, Singapore University of Technology and Design. Prior to that, he was a researcher with Alibaba-NTU Joint Research Institute, Singapore. He received the PhD degree in Nanyang Technological University, Singapore. He was the visiting scholar at Princeton University and University of Waterloo. His research interests include wireless communications, network games and economics, blockchain, and edge intelligence. He has published more than 120 research papers in leading journals and flagship conferences and 6 of them are ESI Highly Cited Papers. He has won several Best Paper Awards in international conferences and technical committee. He is now serving as the editor or guest editor for many leading journals including IEEE JSAC, TVT, IoTJ, TCCN, and TNSE. He is the recipient of the Chinese Government Award for Outstanding Students Abroad in 2019, and NTU SCSE Best PhD Thesis Runner-Up Award in 2020. He is the Founding Vice Chair of Special Interest Group on Wireless Blockchain Networks in IEEE Cognitive Networks Technical Committee.

\end{IEEEbiography}

\begin{IEEEbiography}[{\includegraphics[width=1in,height=1.25in,clip,keepaspectratio]{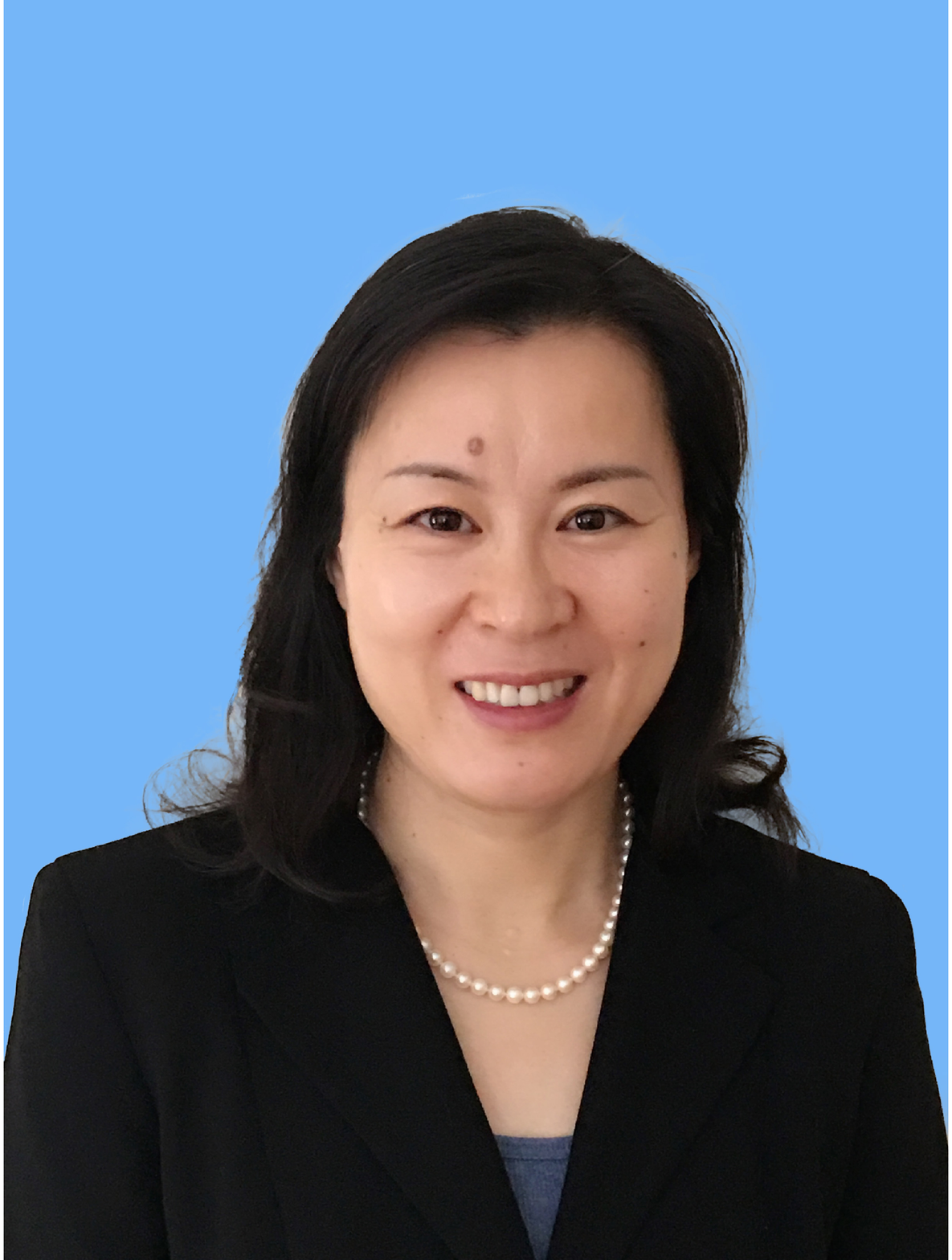}}]{Xiuzhen Cheng}
received her M.S. and Ph.D.
degrees in computer science from the University of
Minnesota Twin Cities in 2000 and 2002. She is currently a
professor in the School
of Computer Science and Technology, Shandong University, China. 
Her current research interests focus on privacy-aware
computing, wireless and mobile security, dynamic
spectrum access, mobile handset networking systems
(mobile health and safety), cognitive radio networks,
and algorithm design and analysis. She has served
on the Editorial Boards of several technical publications and the Technical Program Committees of various professional
conferences/workshops. She has also chaired several international conferences.
She worked as a program director for the U.S. National Science Foundation
(NSF) from April to October 2006 (full time), and from April 2008 to May
2010 (part time). She published more than 300 peer-reviewed papers.
\end{IEEEbiography}

\end{document}